\documentclass{article}
\usepackage{graphicx}
\usepackage{subfigure} 
\usepackage[round]{natbib}
\usepackage{algorithm}
\usepackage{algorithmic}
\usepackage{hyperref}

\usepackage{amsmath,amssymb,amsthm,dsfont,xspace,units}
\usepackage[margin=1.1in]{geometry}
\usepackage{authblk}

\newcommand\wh{\widehat}
\newcommand\wt{\widetilde}

\newcommand\R{\ensuremath{\mathbb{R}}}
\newcommand\N{\ensuremath{\mathbb{N}}}
\newcommand\D{\ensuremath{\mathcal{D}}}
\newcommand\E{\ensuremath{\mathbb{E}}}
\newcommand\one{\ensuremath{\mathds{1}}}
\newcommand\ind[1]{\ensuremath{\mathds1\{#1\}}}
\DeclareMathOperator*{\argmax}{arg\,max}
\DeclareMathOperator{\poly}{poly}
\newcommand\RE[2]{\ensuremath{\operatorname{RE}\left({#1} \| {#2}\right)}}

\DeclareMathOperator{\var}{var}
\newcommand\init{\ensuremath{{\operatorname{init}}}}

\let\originaleqref\eqref
\renewcommand\eqref[1]{Eq.~\originaleqref{eq:#1}}

\newcommand\secref[1]{Section~\ref{sec:#1}}
\newcommand\appref[1]{Appendix~\ref{app:#1}}
\newcommand\algref[1]{Algorithm~\ref{alg:#1}}
\newcommand\stepref[1]{Step~\ref{step:#1}}
\newcommand\thmref[1]{Theorem~\ref{thm:#1}}
\newcommand\lemref[1]{Lemma~\ref{lem:#1}}

\newcommand\Reg{\ensuremath{\operatorname{Reg}}}
\renewcommand\Re{\ensuremath{\mathcal{R}}}

\newcommand\Vmax[1]{\ensuremath{\mathcal{V}_{#1}}}
\newcommand\piopt{\ensuremath{\pi_{\star}}}

\newcommand\optprob{\textsc{OP}}
\newcommand\RUCB{\ensuremath{\mathsf{Randomized\,UCB}}\xspace}
\newcommand\ERUCB{\ensuremath{\mathsf{I LOVE TO CON BANDITS}}\xspace}

\newcommand\SAMPLE{\ensuremath{\mathsf{Sample}}\xspace}
\newcommand\IPS{\ensuremath{\mathsf{IPS}}\xspace}
\newcommand\AMO{\ensuremath{\mathsf{AMO}}\xspace}

\newcommand{\pot}[1]{{\Phi_{#1}}}
\newcommand{\Qm}{\ensuremath{Q^\mu}}
\newcommand{\Qmc}{\ensuremath{Q_c^\mu}}
\newcommand{\Qpm}{\ensuremath{{Q'}^\mu}}
\newcommand{\unifA}{{{\cal U}_A}}
\newcommand{\bpi}{{b_\pi}}
\newcommand{\empexpx}{{\wh{\E}_{x\sim H_t}}}

\newcommand{\zerovec}{{\bf 0}}
\newcommand{\brackets}[1]{{\left[{#1}\right]}}
\newcommand{\paren}[1]{{\left({#1}\right)}}
\newcommand{\braces}[1]{{\left\{#1\right\}}}

\newcommand{\1}{\ind}
\newcommand{\varp}[2]{{V_{#1}({#2})}}
\newcommand{\varsq}[2]{{S_{#1}({#2})}}
\newcommand{\deriv}[2]{{D_{#1}({#2})}}
\newcommand{\updatep}[2]{{\alpha_{#1}({#2})}}

\newtheorem{theorem}{Theorem}
\newtheorem{lemma}{Lemma}
\newtheorem{definition}{Definition}
\newcommand{\lowvar}[1]{\ensuremath{\mathcal{Q}_{#1}}}
\newcommand{\supp}{\ensuremath{\mathrm{supp}}}

\newcommand\GoodEvent{\ensuremath{\mathcal{E}}}
\newcommand\vlc{\ensuremath{\psi}} 
\newcommand\vlcvalue{\ensuremath{100}} 
\newcommand\ratio{{\ensuremath{\rho}}} 
\newcommand\bigc{{\ensuremath{c_0}}}
\newcommand\biggerc{{\ensuremath{C_0}}}

\newcommand{\otil}{\ensuremath{\tilde{O}}}
\newcommand{\logpd}{\ensuremath{\ln(|\Pi|/\delta)}}

\title{Taming the Monster: \\
A Fast and Simple Algorithm for Contextual Bandits}

\author[1]{\mbox{Alekh Agarwal}}
\author[2]{\mbox{Daniel Hsu}}
\author[3]{\mbox{Satyen Kale}}
\author[1]{\mbox{John Langford}}
\author[1]{\mbox{Lihong Li}}
\author[1,4]{\mbox{Robert E.~Schapire}}
\affil[1]{Microsoft Research}
\affil[2]{Columbia University}
\affil[3]{Yahoo!\ Labs}
\affil[4]{Princeton University}

\begin{document} 

\maketitle

\begin{abstract} 
We present a new algorithm for the contextual bandit learning problem,
where the learner repeatedly takes one of $K$ \emph{actions} in
response to the observed \emph{context}, and observes the
\emph{reward} only for that chosen action.
Our method assumes access to an oracle for solving fully supervised
cost-sensitive classification problems and achieves the statistically
optimal regret guarantee with only $\otil(\sqrt{KT/\log N})$ oracle
calls across all $T$ rounds, where $N$ is the number of policies in
the policy class we compete against.
By doing so, we obtain the most practical contextual bandit learning
algorithm amongst approaches that work for general policy classes.
We further conduct a proof-of-concept experiment which demonstrates
the excellent computational and prediction performance of (an online
variant of) our algorithm relative to several baselines.
\end{abstract}

\section{Introduction}

In the contextual bandit problem, an agent collects rewards for
actions taken over a sequence of rounds; in each round, the agent
chooses an action to take on the basis of (i) \emph{context} (or
features) for the current round, as well as (ii) \emph{feedback}, in
the form of rewards, obtained in previous rounds.  The feedback is
\emph{incomplete}: in any given round, the agent observes the reward
only for the chosen action; the agent does not observe the reward for
other actions.  Contextual bandit problems are found in many important
applications such as online recommendation and clinical
trials, and represent a natural half-way point between supervised
learning and reinforcement learning.  The use of features to encode
context is inherited from supervised machine learning, while
\emph{exploration} is necessary for good performance as in
reinforcement learning.


The choice of exploration distribution on actions is important.  The
strongest known results~\citep{Exp4,MS09,Exp4p} provide algorithms
that carefully control the exploration distribution to achieve an
optimal regret after $T$ rounds of \[ O\left(
\sqrt{KT\log(|\Pi|/\delta)} \right), \] with probability at least $1 -
\delta$, relative to a set of policies $\Pi \subseteq A^X$ mapping
contexts $x \in X$ to actions $a \in A$ (where $K$ is the number of
actions).  The regret is the difference between the cumulative reward
of the best policy in $\Pi$ and the cumulative reward collected by the
algorithm.  Because the bound has a mild logarithmic dependence on
$|\Pi|$, the algorithm can compete with very large policy classes that
are likely to yield high rewards, in which case the algorithm also
earns high rewards.  However, the computational complexity of the
above algorithms is linear in $|\Pi|$, making them tractable for only
simple policy classes.

A sub-linear in $|\Pi|$ running time is possible for policy classes
that can be efficiently searched.  In this work, we use the
abstraction of an optimization oracle to capture this property: given
a set of context/reward vector pairs, the oracle returns a policy in
$\Pi$ with maximum total reward.  Using such an oracle in an
i.i.d.~setting (formally defined in \secref{setting}), it is possible
to create $\epsilon$-greedy~\citep{sutton-barto} or
epoch-greedy~\citep{Epoch} algorithms that run in time $O(\log|\Pi|)$
with only a single call to the oracle per round.  However, these
algorithms have suboptimal regret bounds of $O( (K \log|\Pi|)^{1/3}
T^{2/3} )$ because the algorithms randomize uniformly over actions when
they choose to explore.


The \RUCB\ algorithm of \citet{Monster} achieves the optimal regret
bound (up to logarithmic factors) in the i.i.d.~setting, and runs in
time $\poly(T,\log|\Pi|)$ with $\otil(T^5)$ calls to the optimization
oracle per round. Naively this would amount to $\otil(T^6)$ calls to
the oracle over $T$ rounds, although a doubling trick from our
analysis can be adapted to ensure only $\otil(T^5)$ calls to the
oracle are needed over all $T$ rounds in the \RUCB\ algorithm. This is
a fascinating result because it shows that the oracle can provide an
exponential speed-up over previous algorithms with optimal regret
bounds. However, the running time of this algorithm is still
prohibitive for most natural problems owing to the $\otil(T^5)$
scaling. 

In this work, we prove the following\footnote{Throughout this paper,
  we use the $\tilde{O}$ notation to suppress dependence on
  logarithmic factors in $T$ and $K$, as well as $\log(|\Pi|/\delta)$
  (i.e. terms which are $O(\log \log (|\Pi|/\delta))$.}:
\begin{theorem}
  There is an algorithm for the i.i.d.~contextual bandit problem with
  an optimal regret bound requiring
  $\tilde{O}\left(\sqrt{\frac{KT}{\ln(|\Pi|/\delta)}}\right)$ calls to
  the optimization oracle over $T$ rounds, with probability at least
  $1 - \delta$.
\end{theorem}
Concretely, we make $\otil(\sqrt{KT/\logpd})$ calls to the
oracle with a net running time of $\otil(T^{1.5}\sqrt{K\log|\Pi|})$,
vastly improving over the complexity of \RUCB. The major components of
the new algorithm are (i) a new coordinate descent procedure for
computing a very sparse distribution over policies which can be
efficiently sampled from, and (ii) a new epoch structure which allows
the distribution over policies to be updated very infrequently.
We consider variants of the epoch structure that make different
computational trade-offs; on one extreme we concentrate the entire
computational burden on $O(\log T)$ rounds with
$\otil(\sqrt{KT/\logpd})$ oracle calls each time, while on
the other we spread our computation over $\sqrt{T}$ rounds with
$\otil(\sqrt{K/\logpd})$ oracle calls for each of these
rounds.  We stress that in either case, the total number of calls to
the oracle is only sublinear in $T$.  Finally, we develop a more
efficient online variant, and conduct a proof-of-concept experiment
showing low computational complexity and high reward relative to
several natural baselines.

\paragraph{Motivation and related work.}
The EXP4-family of algorithms \citep{Exp4,MS09,Exp4p} solve the
contextual bandit problem with optimal regret by updating weights
(multiplicatively) over all policies in every round.  Except for a few
special cases~\citep{decision-tree,Exp4p}, the running time of such
\emph{measure-based} algorithms is generally linear in the number of
policies.

In contrast, the \RUCB algorithm of \citet{Monster} is based on a
natural abstraction from supervised learning---the ability to
efficiently find a function in a rich function class that minimizes
the loss on a training set.  This abstraction is encapsulated in the
notion of an optimization oracle, which is also useful for
$\epsilon$-greedy~\citep{sutton-barto} and epoch-greedy~\citep{Epoch}
algorithms. However, these latter algorithms have only suboptimal
regret bounds.

Another class of approaches based on Bayesian updating is Thompson
sampling~\citep{Thompson,ThompsonContext}, which often enjoys strong
theoretical guarantees in expectation over the prior and good
empirical performance~\citep{ThompsonExp}. Such algorithms, as well as
the closely related upper-confidence bound algorithms
\citep{Linear,LinUCB}, are computationally tractable in cases where
the posterior distribution over policies can be efficiently maintained
or approximated. In our experiments, we compare to a strong baseline
algorithm that uses this approach~\citep{LinUCB}.



To circumvent the $\Omega(|\Pi|)$ running time barrier, we restrict
attention to algorithms that only access the policy class via the
optimization oracle. Specifically, we use a cost-sensitive
classification oracle, and a key challenge is to design good
supervised learning problems for querying this oracle. The
\RUCB algorithm of~\citet{Monster} uses a similar oracle
to construct a distribution over policies that solves a certain convex
program. However, the number of oracle calls in their work
is prohibitively large, and the statistical analysis is also rather
complex.\footnote{The paper of~\citet{Monster} is colloquially
  referred to, by its authors, as the ``monster
  paper''~\citep{hunch}.}


\paragraph{Main contributions.}
In this work, we present a new and simple algorithm for solving a
similar convex program as that used by \RUCB.  The new algorithm is
based on coordinate descent: in each iteration, the algorithm calls
the optimization oracle to obtain a policy; the output is a sparse
distribution over these policies.  The number of iterations required
to compute the distribution is small---at most
$\tilde{O}(\sqrt{Kt/\logpd})$ in any round $t$.  In fact,
we present a more general scheme based on epochs and warm start in
which the total number of calls to the oracle is, with high
probability, just $\tilde{O}(\sqrt{KT/\logpd})$ \emph{over
  all $T$ rounds}; we prove that this is nearly optimal for a certain
class of optimization-based algorithms.  The algorithm is natural and
simple to implement, and we provide an arguably simpler analysis than
that for \RUCB.  Finally, we report proof-of-concept experimental
results using a variant algorithm showing strong empirical
performance.


\section{Preliminaries}
\label{sec:prelim}

In this section, we recall the i.i.d.~contextual bandit setting and
some basic techniques used in previous
works~\citep{Exp4,Exp4p,Monster}.

\subsection{Learning Setting}
\label{sec:setting}

Let $A$ be a finite set of $K$ actions, $X$ be a space of possible
contexts (\emph{e.g.}, a feature space), and $\Pi \subseteq A^X$ be a
finite set of policies that map contexts $x \in X$ to actions $a \in
A$.\footnote{Extension to VC classes is simple using standard
  arguments.} Let $\Delta^\Pi := \{ Q \in \R^\Pi : Q(\pi) \geq 0 \,
\forall \pi \in \Pi , \, \sum_{\pi \in \Pi} Q(\pi) \leq 1 \}$ be the
set of non-negative weights over policies with total weight at most
one, and let $\R_+^A := \{ r \in \R^A : r(a) \geq 0 \, \forall a \in A
\}$ be the set of non-negative reward vectors.

Let $\D$ be a probability distribution over $X \times [0,1]^A$, the joint
space of contexts and reward vectors; we assume actions' rewards from $\D$
are always in the interval $[0,1]$.
Let $\D_X$ denote the marginal distribution of $\D$ over $X$.

In the i.i.d.~contextual bandit setting, the context/reward vector
pairs $(x_t,r_t) \in X \times [0,1]^A$ over all rounds $t = 1, 2,
\dotsc$ are randomly drawn independently from $\D$.  In round $t$, the
agent first observes the context $x_t$, then (randomly) chooses an
action $a_t \in A$, and finally receives the reward $r_t(a_t) \in
[0,1]$ for the chosen action.  The (observable) record of interaction
resulting from round $t$ is the quadruple $(x_t,a_t,r_t(a_t),p_t(a_t))
\in X \times A \times [0,1] \times [0,1]$; here, $p_t(a_t) \in [0,1]$
is the probability that the agent chose action $a_t \in A$.  We let
$H_t \subseteq X \times A \times [0,1] \times [0,1]$ denote the
\emph{history} (set) of interaction records in the first $t$
rounds. We use the shorthand notation $\empexpx[\cdot]$ to denote
expectation when a context $x$ is chosen from the $t$ contexts in
$H_t$ uniformly at random.

Let $\Re(\pi) := \E_{(x,r) \sim \D}[r(\pi(x))]$ denote the expected
(instantaneous) reward of a policy $\pi \in \Pi$, and let $\piopt :=
\argmax_{\pi \in \Pi} \Re(\pi)$ be a policy that maximizes the
expected reward (the \emph{optimal policy}).  Let $\Reg(\pi) :=
\Re(\piopt) - \Re(\pi)$ denote the \emph{expected (instantaneous)
  regret} of a policy $\pi \in \Pi$ relative to the optimal policy.
Finally, the (empirical cumulative) regret of the agent after $T$
rounds\footnote{We have defined empirical cumulative regret as being
  relative to $\piopt$, rather than to the empirical reward maximizer
  $\argmax_{\pi \in \Pi} \sum_{t=1}^T r_t(\pi(x_t))$.  However, in the
  i.i.d.~setting, the two do not differ by more than
  $O(\sqrt{T\ln(|\Pi|/\delta)})$ with probability at least
  $1-\delta$.}  is defined as
\begin{equation*}
  \sum_{t=1}^T \bigl( r_t(\piopt(x_t)) - r_t(a_t) \bigr) .
\end{equation*}

\subsection{Inverse Propensity Scoring}
\label{sec:ips}

An unbiased estimate of a policy's reward may be obtained from a
history of interaction records $H_t$ using \emph{inverse propensity
  scoring} ($\IPS$; also called \emph{inverse probability weighting}):
the expected reward of policy $\pi \in \Pi$ is estimated as
\begin{equation}
  \label{eq:estimated-reward}
  \wh\Re_t(\pi) := \frac1t
  \sum_{i=1}^t
  \frac{r_i(a_i) \cdot \ind{\pi(x_i) = a_i}}{p_i(a_i)}
  .
\end{equation}
This technique can be viewed as mapping $H_t \mapsto \IPS(H_t)$ of
interaction records $(x,a,r(a),p(a))$ to context/reward vector pairs
$(x,\hat{r})$, where $\hat{r} \in \R_+^A$ is a fictitious reward
vector that assigns to the chosen action $a$ a scaled reward
$r(a)/p(a)$ (possibly greater than one), and assigns to all other
actions zero rewards.  This transformation $\IPS(H_t)$ is detailed in
\algref{IPS} (in \appref{details}); we may equivalently define
$\wh\Re_t$ by $\wh\Re_t(\pi) := t^{-1} \sum_{(x,\hat{r}) \in
  \IPS(H_t)} \hat{r}(\pi(x))$.  It is easy to verify that $\E[
  \hat{r}(\pi(x)) | (x,r) ] = r(\pi(x))$, as $p(a)$ is indeed the
agent's probability (conditioned on $(x,r)$) of picking action $a$.
This implies $\wh\Re_t(\pi)$ is an unbiased estimator for any history
$H_t$.

Let $\pi_t := \argmax_{\pi \in \Pi} \wh\Re_t(\pi)$ denote a policy
that maximizes the expected reward estimate based on inverse
propensity scoring with history $H_t$ ($\pi_0$ can be arbitrary), and
let $\wh\Reg_t(\pi) := \wh\Re_t(\pi_t) - \wh\Re_t(\pi)$ denote
\emph{estimated regret} relative to $\pi_t$.  Note that
$\wh\Reg_t(\pi)$ is generally \emph{not} an unbiased estimate of
$\Reg(\pi)$, because $\pi_t$ is not always $\piopt$.

\subsection{Optimization Oracle}
\label{sec:opt-oracle}

One natural mode for accessing the set of policies $\Pi$ is
enumeration, but this is impractical in general.  In this work, we
instead only access $\Pi$ via an optimization oracle which
corresponds to a cost-sensitive learner.  Following \citet{Monster},
we call this oracle $\AMO$\footnote{Cost-sensitive learners often need
  a cost instead of reward, in which case we use $c_t = \one - r_t$.}.
\begin{definition}
  For a set of policies $\Pi$, the $\argmax$ oracle ($\AMO$) is an algorithm, which
for any sequence of context and reward vectors, $(x_1, r_1), (x_2, r_2),
\ldots, (x_t, r_t) \in X \times \R_+^A$, returns
\[ \argmax_{\pi \in \Pi} \sum_{\tau = 1}^t r_\tau(\pi(x_\tau)) . \]
\end{definition}

\subsection{Projections and Smoothing}
\label{sec:proj-smoothing}

In each round, our algorithm chooses an action by randomly drawing a
policy $\pi$ from a distribution over $\Pi$, and then picking the
action $\pi(x)$ recommended by $\pi$ on the current context $x$. This
is equivalent to drawing an action according to $Q(a|x) := \sum_{\pi
  \in \Pi : \pi(x) = a} Q(\pi) , \, \forall a \in A$. For keeping
the variance of reward estimates from $\IPS$ in check, it is desirable
to prevent the probability of any action from being too small. Thus,
as in previous work, we also use a smoothed projection
$Q^\mu(\cdot|x)$ for $\mu \in [0,1/K]$, $ Q^\mu(a|x) := (1-K\mu)
\sum_{\pi \in \Pi : \pi(x) = a} Q(\pi) + \mu , \, \forall a \in
A$. Every action has probability at least $\mu$ under
$Q^\mu(\cdot|x)$.

For technical reasons, our algorithm maintains non-negative weights $Q
\in \Delta^\Pi$ over policies that sum to at most one, but not
necessarily equal to one; hence, we put any remaining mass on a
default policy $\bar\pi \in \Pi$ to obtain a legitimate probability
distribution over policies $\tilde{Q} = Q + \left(1 - \sum_{\pi \in
  \Pi} Q(\pi) \right)\one_{\bar\pi}$. We then pick an action from the
smoothed projection $\tilde{Q}^\mu(\cdot|x)$ of $\tilde{Q}$ as above.
This sampling procedure $\SAMPLE(x,Q,\bar\pi,\mu)$ is detailed in
\algref{SAMPLE} (in \appref{details}).

\section{Algorithm and Main Results}
\label{sec:alg}

Our algorithm (\ERUCB) is an epoch-based variant of the \RUCB
algorithm of \citet{Monster} and is given in \algref{erucb}.  Like
\RUCB, \ERUCB solves an optimization problem (\optprob) to obtain
a distribution over policies to sample from (\stepref{solve-op}), but
does so on an {\em epoch schedule}, \emph{i.e.}, only on certain
pre-specified rounds $\tau_1, \tau_2, \ldots$.  The only requirement
of the epoch schedule is that the length of epoch $m$ is bounded as
$\tau_{m+1} - \tau_m = O(\tau_m)$.  For simplicity, we assume
$\tau_{m+1} \leq 2\tau_m$ for $m \geq 1$, and $\tau_1 = O(1)$.

The crucial step here is solving (\optprob). Before stating the main
result, let us get some intuition about this problem. The first
constraint, \eqref{reg-cons}, requires the average estimated
regret of the distribution $Q$ over policies to be small, since
$b_\pi$ is a rescaled version of the estimated regret of policy $\pi$.
This
constraint skews our distribution to put more mass on ``good
policies'' (as judged by our current information), and can be seen as
the exploitation component of our algorithm. The second set of
constraints, \eqref{var-cons}, requires the
distribution $Q$ to place sufficient mass on the actions chosen by each
policy $\pi$, in expectation over contexts. This can be thought of as
the exploration constraint, since it requires the distribution to
be sufficiently diverse for most contexts. As we will see later, the
left hand side of the constraint is a bound on the variance of our
reward estimates for policy $\pi$, and the constraint requires the
variance to be controlled at the level of the estimated regret of $\pi$.
That is, we require the reward estimates to be more accurate for
good policies than we do for bad ones, allowing for much more adaptive
exploration than the uniform exploration of $\epsilon$-greedy style
algorithms.

This problem is very similar to the one in \citet{Monster}, and our
coordinate descent algorithm in \secref{coord} gives a constructive
proof that the problem is feasible. As in \citet{Monster}, we have the
following regret bound:
\begin{theorem} \label{thm:regret-main}
  Assume the optimization problem (\optprob) can be solved whenever
  required in Algorithm~\ref{alg:erucb}.
   With probability at least $1-\delta$, the regret of
   Algorithm~\ref{alg:erucb} (\ERUCB) after $T$ rounds is
  \begin{equation*}
  O\left( \sqrt{KT\ln(T|\Pi|/\delta)} + K\ln(T|\Pi|/\delta) \right)
  .
  \end{equation*}
\end{theorem}

\begin{algorithm}[h]
  \caption{Importance-weighted LOw-Variance Epoch-Timed Oracleized
    CONtextual BANDITS algorithm (\ERUCB)}
  \label{alg:erucb}
  \begin{algorithmic}[1]
    \renewcommand{\algorithmicrequire}{\textbf{input}}

    \REQUIRE Epoch schedule $0 = \tau_0 < \tau_1 < \tau_2 < \dotsb$, allowed failure probability $\delta \in (0,1)$.

    \STATE Initial weights $Q_0 := \zerovec \in \Delta^\Pi$, initial
    epoch $m := 1$. \\
    Define $\mu_m := \min\{ \nicefrac{1}{2K},
    \sqrt{\ln(16\tau_m^2|\Pi|/\delta)/(K\tau_m)} \}$ for all $m \geq 0$.

    \FOR{\textbf{round} $t = 1, 2, \dotsc$}

      \STATE Observe context $x_t \in X$.

      \STATE $(a_t,p_t(a_t)) :=
      \SAMPLE(x_t,Q_{m-1},\pi_{\tau_m-1},\mu_{m-1})$.
      \label{step:sampling}

      \STATE Select action $a_t$ and observe reward $r_t(a_t) \in [0,1]$.

      \IF{$t = \tau_m$}

        \STATE Let $Q_m$ be a solution to (\optprob) with history $H_t$
        and minimum probability $\mu_m$.
        \label{step:solve-op}

        \STATE $m := m + 1$.

      \ENDIF

    \ENDFOR
  \end{algorithmic}
\end{algorithm}

\begin{figure}[h]
\begin{center} \fbox{\begin{minipage}{0.95\columnwidth} 
\centerline{{\bf Optimization Problem} (\optprob)}
  Given a history $H_t$ and minimum probability $\mu_m$,
  define $\bpi := \frac{\wh\Reg_t(\pi)}{\vlc \mu_m}$ for $\vlc := \vlcvalue$,
  and find $Q \in \Delta^\Pi$ such that
    \begin{align}
      \sum_{\pi \in \Pi} Q(\pi) \bpi\ &\leq\ 2K \label{eq:reg-cons}\\
      \forall \pi \in \Pi:\ \empexpx\left[\frac{1}{Q^{\mu_m}(\pi(x) |
          x)}\right]\ &\leq\ 2K + \bpi. \label{eq:var-cons}
    \end{align}
\end{minipage}}
\end{center}
\end{figure}

\subsection{Solving (\optprob) via Coordinate Descent}
\label{sec:coord}

We now present a coordinate descent algorithm to solve (\optprob). The
pseudocode is given in \algref{coord}. Our analysis, as well as the
algorithm itself, are based on a potential function which we use to
measure progress. The algorithm can be viewed as a form of coordinate
descent applied to this same potential function. The main idea of our
analysis is to show that this function decreases substantially on
every iteration of this algorithm; since the function is nonnegative,
this gives an upper bound on the total number of iterations as
expressed in the following theorem.
\begin{theorem} \label{thm:coord-coldstart}
\algref{coord} (with $Q_\init := \zerovec$) halts in
at most $\frac{4 \ln(1/(K\mu_m))}{\mu_m}$
iterations, and outputs a solution $Q$ to (\optprob).
\end{theorem}

\begin{algorithm}[h]
  \caption{Coordinate Descent Algorithm}
  \label{alg:coord}
  \begin{algorithmic}[1]
    \REQUIRE History $H_t$, minimum probability $\mu$, initial weights $Q_\init \in \Delta^\Pi$.

    \STATE Set $Q := Q_\init$.
    \LOOP
    \STATE \label{step:definitions} Define, for all $\pi \in \Pi$,
        \begin{eqnarray*}
          \varp{\pi}{Q} &=& \empexpx\brackets{1/\Qm(\pi(x)|x)} \\
          \varsq{\pi}{Q} &=& \empexpx\brackets{1/(\Qm(\pi(x)|x))^2} \\
          \deriv{\pi}{Q} &=& \varp{\pi}{Q} - (2K + b_\pi).
        \end{eqnarray*}

        \IF{$\sum_\pi Q(\pi) (2K+\bpi) > 2K$}  \label{step:1}
      \STATE Replace $Q$ by $cQ$, where
        \begin{equation}  \label{eq:d3}
          c := \frac{2K}{\sum_\pi  Q(\pi) (2K+\bpi)} < 1 .
        \end{equation}
    \ENDIF

    \IF{there is a policy $\pi$ for which $\deriv{\pi}{Q} > 0$} \label{step:violating-policy}
      \STATE \label{step:2b}
         Add the (positive) quantity
         \[
            \updatep{\pi}{Q} = \frac{\varp{\pi}{Q} + \deriv{\pi}{Q}}
                             {2(1-K\mu) \varsq{\pi}{Q}}
         \]
         to $Q(\pi)$ and leave all other weights unchanged.
    \ELSE
      \STATE \label{step:2a}
         Halt and output the current set of weights $Q$.
    \ENDIF
    \ENDLOOP
\end{algorithmic}
\end{algorithm}

\subsection{Using an Optimization Oracle}

We now show how to implement \algref{coord} via $\AMO$ (c.f.~\secref{opt-oracle}). 
\begin{lemma}
  \algref{coord} can be implemented using one call to $\AMO$ before
  the loop is started, and one call for each iteration of the loop
  thereafter.
  \label{lemma:amo-reduction}
\end{lemma}
\begin{proof}
At the very beginning, before the loop is started, we compute the best
empirical policy so far, $\pi_t$, by calling $\AMO$ on the sequence of
historical contexts and estimated reward vectors; \emph{i.e.}, on $(x_\tau,
\hat{r}_\tau)$, for $\tau = 1, 2, \ldots, t$.

Next, we show that each iteration in the loop of \algref{coord}
can be implemented via one call to $\AMO$.
Going over the pseudocode, first note that operations involving $Q$ in
\stepref{1} can be performed efficiently since $Q$ has sparse
support.
Note that the definitions in \stepref{definitions} don't actually
need to be computed for all policies $\pi \in \Pi$, as long as we can
identify a policy $\pi$ for which $D_\pi(Q) > 0$.
We can identify such a policy using one call to $\AMO$ as follows.

First, note that for any policy $\pi$, we have
\[
  V_\pi(Q)
  = \empexpx\brackets{\frac{1}{\Qm(\pi(x)|x)}}
  = \frac{1}{t}\sum_{\tau=1}^t \frac{1}{\Qm( \pi(x_\tau)|x_\tau)} ,
\]
and
\[
  b_\pi
  = \frac{\wh\Reg_t(\pi)}{\vlc \mu}
  = \frac{\wh\Re_t(\pi_t)}{\vlc \mu}
    - \frac{1}{\vlc \mu t}\sum_{\tau=1}^t \hat{r}_\tau(\pi(x_\tau)) .
\]
Now consider the sequence of historical contexts and reward vectors,
$(x_\tau, \tilde{r}_\tau)$ for $\tau = 1, 2, \ldots, t$, where for any
action $a$ we define
\begin{equation} \label{eq:cse}
  \tilde{r}_\tau(a)
  := \frac1t \paren{ \frac{\vlc \mu}{\Qm(a | x_\tau)} + \hat{r}_\tau(a) } .
\end{equation}
It is easy to check that
\[
  D_\pi(Q)
  = \frac1{\vlc\mu} \sum_{\tau=1}^t \tilde{r}_\tau(\pi(x_\tau))
    - \left(2K + \frac{\wh\Re_t(\pi_t)}{\vlc\mu}\right) .
\]
Since $2K + \frac{\wh\Re_t(\pi_t)}{\vlc\mu}$ is a constant independent
of $\pi$, we have 
  \[\arg\max_{\pi \in \Pi} D_\pi(Q) = \arg\max_{\pi \in \Pi}
  \sum_{\tau=1}^t \tilde{r}_\tau(\pi(x_\tau)),\] 
  and hence, calling $\AMO$ once on the sequence $(x_\tau,
  \tilde{r}_\tau)$ for $\tau = 1, 2, \ldots, t$, we obtain a policy
  that maximizes $D_\pi(Q)$, and thereby identify a policy for which
  $D_\pi(Q) > 0$ whenever one exists.
\end{proof}

\subsection{Epoch Schedule}
\label{sec:doubling}

Recalling the setting of $\mu_m$ in Algorithm~\ref{alg:erucb},
\thmref{coord-coldstart} shows that \algref{coord} solves (\optprob)
with $\tilde{O}(\sqrt{Kt/\logpd})$ calls to $\AMO$ in round
$t$.  Thus, if we use the epoch schedule $\tau_m = m$ (\emph{i.e.},
run \algref{coord} in every round), then we get a total of
$\tilde{O}(\sqrt{KT^3/\logpd})$ calls to $\AMO$ over all
$T$ rounds.  This number can be dramatically reduced using a more
carefully chosen epoch schedule.

\begin{lemma}
\label{lem:doubling}
  For the epoch schedule $\tau_m := 2^{m-1}$, the
  total number of calls to $\AMO$ is $\tilde{O}(\sqrt{KT/\logpd})$.
\end{lemma}
\begin{proof}
  The epoch schedule satisfies the requirement $\tau_{m+1} \leq
  2\tau_m$.  With this epoch schedule, \algref{coord} is run only
  $O(\log T)$ times over $T$ rounds, leading to
  $\tilde{O}(\sqrt{KT/\logpd})$ total calls to $\AMO$ over
  the entire period.
\end{proof}

\subsection{Warm Start}
\label{sec:warm-start}

We now present a different technique to reduce the number of calls to
$\AMO$.
This is based on the observation that practically speaking, it seems
terribly wasteful, at the start of a new epoch, to throw out the results of
all of the preceding computations and to begin yet again from nothing.
Instead, intuitively, we expect computations to be more moderate if we
begin again where we left off last, \emph{i.e.}, a ``warm-start'' approach.
Here, when \algref{coord} is called at the end of epoch $m$, we use
$Q_\init := Q_{m-1}$ (the previously computed weights) rather than
$\zerovec$.


We can combine warm-start with a different epoch schedule to guarantee
$\tilde{O}(\sqrt{KT/\logpd})$ total calls to $\AMO$, spread
across $O(\sqrt{T})$ calls to \algref{coord}.
\begin{lemma}
\label{lem:warm-start}
Define the epoch schedule $(\tau_1,\tau_2) := (3,5)$ and $\tau_m :=
m^2$ for $m \geq 3$ (this satisfies $\tau_{m+1} \leq 2\tau_m$).  With
high probability, the warm-start variant of \algref{erucb} makes
$\tilde{O}(\sqrt{KT/\logpd})$ calls to $\AMO$ over $T$
rounds and $O(\sqrt{T})$ calls to \algref{coord}.
\end{lemma}

\subsection{Computational Complexity}
\label{sec:complexity}

So far, we have only considered computational complexity in terms of the
number of oracle calls. However, the reduction also involves the creation
of cost-sensitive classification examples, which must be
accounted for in the net computational cost. As observed in the
proof of Lemma~\ref{lemma:amo-reduction} (specifically~\eqref{cse}),
this requires the computation of the probabilities $\Qm(a | x_\tau)$
for $\tau = 1,2,\ldots, t$ when the oracle has to be invoked at round
$t$. According to Lemma~\ref{lem:warm-start}, the support of the
distribution $Q$ at time $t$ can be over at most
$\tilde{O}(\sqrt{Kt/\logpd})$ policies (same as the number of calls to
$\AMO$). This would suggest a computational complexity of
$\otil(\sqrt{Kt^3/\logpd})$ for querying the oracle at time $t$,
resulting in an overall computation cost scaling with $T^2$.

We can, however, do better with some natural bookkeeping.
Observe that at the start of round $t$, the conditional distributions $Q(a
| x_i)$ for $i=1,2,\dotsc,t-1$ can be represented as a table of size
$K\times (t-1)$, where rows and columns correspond to actions and contexts.
Upon receiving the new example in round $t$, the corresponding
$t$-th column can be added to this table in time
$K \cdot |\supp(Q)| = \tilde{O}(K\sqrt{Kt/\logpd})$ (where
$\supp(Q) \subseteq \Pi$ denotes the support of $Q$), using
the projection operation described in
Section~\ref{sec:proj-smoothing}. Hence the net cost of these
updates, as a function of $K$ and $T$, scales with as $(KT)^{3/2}$.
Furthermore, the
cost-sensitive examples needed for the $\AMO$ can be obtained by a
simple table lookup now, since the action probabilities are directly
available. This involves $O(Kt)$ table lookups when the oracle is
invoked at time $t$, and again results in an overall cost
scaling as $(KT)^{3/2}$. Finally, we have to update the table when the
distribution $Q$ is updated in Algorithm~\ref{alg:coord}. If we find
ourselves in the rescaling step 4, we can simply store the constant
$c$. When we enter step 8 of the algorithm, we can do a linear scan
over the table, rescaling and incrementing the entries. This also
resutls in a cost of $O(Kt)$ when the update happens at time
$t$, resulting in a net scaling as $(KT)^{3/2}$. Overall, we find that
the computational complexity of our algorithm, modulo the oracle
running time, is $\tilde{O}(\sqrt{(KT)^3/\logpd})$.

\subsection{A Lower Bound on the Support Size}
\label{sec:oracle-lb-result}

An attractive feature of the coordinate descent algorithm,
\algref{coord}, is that the number of oracle calls is directly
related to the number of policies in the support of $Q_m$.
Specifically, for the doubling schedule of \secref{doubling},
\thmref{coord-coldstart} implies that we never have non-zero
weights for more than $\frac{4 \ln(1/(K\mu_m))}{\mu_m}$ policies in epoch
$m$.
Similarly, the total number of oracle calls for the warm-start approach in
\secref{warm-start} bounds the total number of policies which ever have
non-zero weight over all $T$ rounds.
The support size of the distributions $Q_m$ in \algref{erucb} is
crucial to the computational complexity of sampling an action
(\stepref{sampling} of \algref{erucb}).

In this section, we demonstrate a lower bound showing that it is not
possible to construct substantially sparser distributions that also satisfy
the low-variance constraint (\ref{eq:var-cons}) in the optimization problem
(\optprob).
To formally define the lower bound,
fix an epoch schedule $0 = \tau_0 < \tau_1 < \tau_2 < \dotsb$ and
consider the following set of
non-negative vectors over policies:
\begin{equation*}
  \lowvar{m} \kern-2pt:=\kern-2pt \{Q \in \Delta^{\Pi} :
  \text{$Q$ satisfies \eqref{var-cons} in round $\tau_m$}\}.
\end{equation*}
(The distribution $Q_m$ computed by \algref{erucb} is in
$\lowvar{m}$.)
Recall that $\supp(Q)$ denotes the support of $Q$ (the set of
policies where $Q$ puts non-zero entries).
We have the following lower bound on $|\supp(Q)|$.
\begin{theorem}
  For any epoch schedule $0 = \tau_0 < \tau_1 < \tau_2 < \dotsb$ and
  any $M \in \N$ sufficiently large,
  there exists a distribution $\D$ over $X\times [0,1]^A$ and a policy
  class $\Pi$ such that,
  with probability at least $1 - \delta$,
  \[
    \inf_{\substack{m \in \N : \\ \tau_m \geq \tau_M/2}}
    \inf_{Q \in \lowvar{m}}
    |\supp(Q)|
    = \Omega\paren{ \sqrt{\frac{K\tau_M}{\ln(|\Pi|\tau_M/\delta)}} } .
  \]
  \label{thm:lb}
\end{theorem}

The proof of the theorem is deferred to \appref{lb}.
In the context of our problem, this lower bound shows that the bounds in
\lemref{doubling} and \lemref{warm-start} are unimprovable, since
the number of calls to $\AMO$ is at least the size of the support, given
our mode of access to $\Pi$.

\section{Regret Analysis}
\label{sec:regret}

In this section, we outline the regret analysis for our algorithm \ERUCB,
with details deferred to \appref{deviation} and \appref{regret}.

The deviations of the policy reward estimates $\wh\Re_t(\pi)$ are
controlled by (a bound on) the variance of each term in
\eqref{estimated-reward}: essentially the left-hand side of
\eqref{var-cons} from (\optprob), except with $\empexpx[\cdot]$ replaced by
$\E_{x \sim \D_X}[\cdot]$.
Resolving this discrepancy is handled using deviation bounds, so
\eqref{var-cons} holds with $\E_{x \sim \D_X}[\cdot]$, with worse
right-hand side constants.

The rest of the analysis, which deviates from that of \RUCB,
compares the expected regret $\Reg(\pi)$ of any policy $\pi$ with
the estimated regret $\wh\Reg_t(\pi)$ using the variance constraints
\eqref{var-cons}:
\begin{lemma}[Informally]
  \label{lem:inductive-informal}
With high probability, for each $m$ such that $\tau_m \geq
\tilde{O}(K\log|\Pi|)$, each round $t$ in epoch $m$, and each $\pi \in \Pi$,
  $\Reg(\pi)
  \leq 2 \wh\Reg_t(\pi)
  + O(K\mu_m)$.
\end{lemma}

This lemma can easily be combined with the constraint \eqref{reg-cons} from
(\optprob): since the weights $Q_{m-1}$ used in any round $t$ in epoch $m$
satisfy $\sum_{\pi \in \Pi} Q_{m-1}(\pi) \wh\Reg_{\tau_m-1}(\pi) \leq \vlc
\cdot 2K \mu_{\tau_m-1}$, we obtain a bound on the (conditionally) expected
regret in round $t$ using the above lemma: with high probability,
\[ \sum_{\pi \in \Pi} \wt{Q}_{m-1} \Reg(\pi) \leq O(K\mu_{m-1}) . \]
Summing these terms up over all $T$ rounds and applying martingale
concentration gives the final regret bound in \thmref{regret-main}.


\newcommand{\Qmm}{\ensuremath{Q^{\mu_{m}}}}
\newcommand{\Qmmp}{\ensuremath{Q^{\mu_{m+1}}}}
\newcommand{\bmpi}{\ensuremath{b^m_\pi}}
\newcommand{\softOh}{{\tilde{O}}}
\newcommand{\softOmega}{{\tilde{\Omega}}}
\newcommand{\pota}[1]{{\phi_{#1}^a}}
\newcommand{\potb}[1]{{\phi_{#1}^b}}
\newcommand{\potc}[1]{{\phi_{#1}^c}}
\newcommand{\potd}[1]{{\phi_{#1}^d}}

\newcommand{\emprew}{\wh\Re}
\newcommand{\empreg}[1]{\wh{\Reg}_{#1}}
\newcommand{\cumR}{\wh{\cal S}}
\newcommand{\instR}{\hat{r}}
\newcommand{\Qdist}{\ensuremath{\tilde{Q}}}
\newcommand{\Qdistmpi}{\ensuremath{\tilde{Q}_m(\pi)}}
\newcommand{\Qdisttautpi}{\ensuremath{\tilde{Q}_{\tau(t)}(\pi)}}
\newcommand{\Qdistmum}{\ensuremath{\tilde{Q}^{\mu_m}}}

\newcommand{\bconst}{\phi}
\newcommand{\bconstinv}{{\bconst^{-1}}}

\let\originalempexpx\empexpx
\renewcommand{\empexpx}{\wh{\E}_x}

\newcommand{\proofsketch}{\paragraph{\it Proof sketch.}}

\section{Analysis of the Optimization Algorithm}

In this section, we give a sketch of the analysis of our main
optimization algorithm for computing weights $Q_m$ on each epoch as in
\algref{coord}. As mentioned in \secref{coord}, this analysis is based
on a potential function.

Since our attention for now is on a single epoch $m$, here and in what
follows, when clear from context, we drop $m$ from our notation and
write simply $\tau=\tau_m$, $\mu=\mu_m$, etc. Let $\unifA$ be the
uniform distribution over the action set $A$. We define the following
potential function for use on epoch $m$:
\begin{equation}
{  \pot{m}(Q)}
=
   \tau\mu \paren{
     \frac{\empexpx\brackets{ \RE{\unifA}{\Qm(\cdot \mid x)}}}{1-K\mu} 
        +
     \frac{\sum_{\pi \in \Pi} Q(\pi) \bpi}{2K}
                 }.
\label{eq:pot}
\end{equation}
The function in \eqref{pot}
is defined for all vectors $Q \in \Delta^\Pi$.
Also, $\RE{p}{q}$ denotes the unnormalized relative entropy between
two nonnegative vectors $p$ and $q$ over the action space (or any set) $A$:
\[
  \RE{p}{q} = \sum_{a\in A} (p_a \ln(p_a/q_a) + q_a - p_a).
\]
This number is always nonnegative.
Here, $\Qm(\cdot | x)$ denotes the ``distribution'' (which might not sum
to $1$) over $A$ induced by $\Qm$ for context $x$ as given in
\secref{proj-smoothing}.
Thus, ignoring constants, this potential function is a combination of
two terms:
The first measures how far from uniform are the distributions induced
by $\Qm$, and the second is an estimate of expected regret under $Q$
since $\bpi$ is proportional to the empirical regret of $\pi$.
Making $\pot{m}$ small thus encourages $Q$ to choose actions
as uniformly as possible while also incurring low regret ---
exactly the aims of our algorithm.
The constants that appear in this definition are for later mathematical
convenience.

For further intuition, note that, by straightforward calculus, the
partial derivative ${\partial \pot{m}}/{\partial Q(\pi)}$ is roughly
proportional to the variance constraint for $\pi$ given in
\eqref{var-cons} (up to a slight mismatch of constants).  This shows
that if this constraint is not satisfied, then $\partial \pot{m} /
\partial Q(\pi)$ is likely to be negative, meaning that $\pot{m}$ can
be decreased by increasing $Q(\pi)$.  Thus, the weight vector $Q$ that
minimizes $\pot{m}$ satisfies the variance constraint for every policy
$\pi$.  It turns out that this minimizing $Q$ also satisfies the low
regret constraint in \eqref{reg-cons}, and also must sum to at most
$1$; in other words, it provides a complete solution to our
optimization problem.  \algref{coord} does not fully minimize
$\pot{m}$, but it is based roughly on coordinate descent. This is
because in each iteration one of the weights (coordinate directions)
$Q(\pi)$ is increased. This weight is one whose corresponding partial
derivative is large and negative.

To analyze the algorithm, we first argue that it is correct in the sense of
satisfying the required constraints, provided that it halts.

\begin{lemma}  \label{lem:opt-correct}
If \algref{coord} halts and outputs a weight vector $Q$,
then the constraints \eqref{var-cons} and
\eqref{reg-cons} must hold, and furthermore
the sum of the weights $Q(\pi)$ is at most $1$.
\end{lemma}

The proof is rather straightforward:
Following \stepref{1}, \eqref{reg-cons} must hold, and also the
weights must sum to $1$.
And if the algorithm halts, then $\deriv{\pi}{Q}\leq 0$ for all $\pi$,
which is equivalent to \eqref{var-cons}.

What remains is the more challenging task of bounding the number of
iterations until the algorithm does halt.
We do this by showing that significant progress is made in reducing
$\pot{m}$ on every iteration.
To begin, we show that scaling $Q$ as in \stepref{1} cannot cause
$\pot{m}$ to increase.

\begin{lemma}  \label{lem:scale-pot}
Let $Q$ be a weight vector such that $\sum_\pi Q(\pi) (2K+\bpi) > 2K$, and let $c$ be as in \eqref{d3}. Then $\pot{m}(cQ) \leq \pot{m}(Q)$.
\end{lemma}

\proofsketch
We consider $\pot{m}(cQ)$ as a function of $c$,
and argue that its derivative (with respect to $c$) at the
value of $c$ given in the lemma statement is always nonnegative.
Therefore, by convexity, it is nondecreasing for all values exceeding
$c$.
Since $c<1$, this proves the lemma.
\qed

Next, we show that substantial progress will be made in reducing $\pot{m}$
each time that \stepref{2b} is executed.

\begin{lemma}  \label{lem:pot-dec}
Let $Q$ denote a set of weights and suppose, for some policy $\pi$,
that $\deriv{\pi}{Q}>0$.
Let $Q'$ be a new set of weights which is an exact copy of $Q$ except
that 
$  Q'(\pi) = Q(\pi) + \alpha  $
where $\alpha=\updatep{\pi}{Q}>0$.
Then
\begin{equation}  \label{eq:a6b}
   \pot{m}(Q) - \pot{m}(Q') \geq \frac{{\tau} \mu^2} {4 (1-K\mu)}.
\end{equation}
\end{lemma}

\proofsketch
We first compute exactly the change in potential for general $\alpha$.
Next, we apply a second-order Taylor approximation, which is maximized
by the $\alpha$ used in the algorithm.
The Taylor approximation, for this $\alpha$, yields a lower bound
which can be further simplified using the fact that $\Qm(a|x)\geq\mu$
always, and our assumption that $\deriv{\pi}{Q}>0$.
This gives the bound stated in the lemma.
\qed

So \stepref{1} does not cause $\pot{m}$ to increase, and
\stepref{2b} causes $\pot{m}$ to decrease by at least the amount given
in \lemref{pot-dec}. This immediately implies \thmref{coord-coldstart}: for
$Q_\init = \zerovec$, the initial potential is bounded by ${\tau\mu
\ln(1/(K\mu))}/{(1-K\mu)}$, and it is never negative, so the number of
times \stepref{2b} is executed is bounded by $4 \ln(1/(K\mu))/\mu$ as required.




\subsection{Epoching and Warm Start}

As shown in \secref{opt-oracle}, the bound on the number of iterations of
the algorithm from \thmref{coord-coldstart} also gives a bound on the
number of times the oracle is called. To reduce the number of oracle calls,
one approach is the ``doubling trick'' of \secref{doubling}, which enables us to bound 
%
the total combined number of iterations of \algref{coord} in the first
$T$ rounds is only $\softOh(\sqrt{KT/\logpd})$. This means that the
average number of calls to the arg-max oracle is only
$\softOh(\sqrt{K/(T\logpd)})$ per round, meaning that the oracle is
called far less than once per round, and in fact, at a vanishingly low
rate.

We now turn to warm-start approach of \secref{warm-start}, where in
each epoch $m+1$ we initialize the coordinate descent algorithm with
$Q_\init = Q_m$, i.e. the weights computed in the previous epoch
$m$. To analyze this, we bound how much the potential changes from
$\pot{m}(Q_m)$ at the end of epoch $m$ to $\pot{m+1}(Q_m)$ at the very
start of epoch $m+1$. This,
combined with our earlier results regarding how quickly \algref{coord} drives down
the potential, we are able to get an overall bound on the total number
of updates across $T$ rounds.


\begin{lemma}  \label{lem:warm-ep2ep-bnd}
Let $M$ be the largest integer for which $\tau_{M+1}\leq T$.  With
probability at least $1-2\delta$, for all $T$, the total
epoch-to-epoch increase in potential is
\[
  \sum_{m=1}^M  (\pot{m+1}(Q_m) - \pot{m}(Q_m))
     \leq \softOh\paren{\sqrt{\frac{T\logpd}{K}}},
\]
where $M$ is the largest integer for which $\tau_{M+1}\leq T$.
\end{lemma}

\proofsketch
The potential function, as written in \eqref{pot}, naturally breaks
into two pieces whose epoch-to-epoch changes can be bounded
separately.
Changes affecting the relative entropy term on the left can be
bounded, regardless of $Q_m$, by taking advantage of the manner in
which these distributions are smoothed.
For the other term on the right, it turns out that these
epoch-to-epoch changes are related to statistical quantities which can
be bounded with high probability.
Specifically, the total change in this term is related first to how
the estimated reward of the empirically best policy compares to the
expected reward of the optimal policy; and second, to how the
reward received by our algorithm compares to that of the optimal
reward.
From our regret analysis, we are able to show that both of these
quantities will be small with high probability.
\qed

This lemma, along with \lemref{pot-dec} can be used to further
establish \lemref{warm-start}. We only provide an intuitive sketch
here, with the details deferred to the appendix. As we observe in
\lemref{warm-ep2ep-bnd}, the total amount that the potential increases
across $T$ rounds is at most $\softOh(\sqrt{T\logpd /K})$.  On the
other hand, \lemref{pot-dec} shows that each time $Q$ is updated by
\algref{coord} the potential decreases by at least
$\softOmega(\logpd/K)$ (using our choice of $\mu$).  Therefore, the
total number of updates of the algorithm totaled over all $T$ rounds
is at most $\softOh(\sqrt{KT/\logpd})$.  For instance, if we use
$(\tau_1,\tau_2) := (3,5)$ and $\tau_m:=m^2$ for $m \geq 3$, then the
weight vector $Q$ is only updated about $\sqrt{T}$ times in $T$
rounds, and on each of those rounds, \algref{coord} requires
$\softOh(\sqrt{K/\logpd})$ iterations, on average, giving the claim in
\lemref{warm-start}.

\renewcommand{\empexpx}{\originalempexpx}

\section{Experimental Evaluation}

\begin{table*}
\begin{center}
\caption{Progressive validation loss, best hyperparameter values, and running
times of various algorithm on RCV1.} \begin{tabular}{|c|c|c|c|c|c||c|}
\hline 
\textbf{Algorithm} & $\epsilon$-greedy & Explore-first & Bagging & LinUCB & Online Cover & Supervised\\
\hline
\textbf{P.V. Loss} & $0.148$ & $0.081$ & $0.059$ & $0.128$ & $0.053$ & $0.051$\\
\hline
\textbf{Searched} & $0.1 = \epsilon$  & $2 \times 10^5$ first & $16$ bags & $10^3$ dim,
minibatch-10 & cover $n = 1$ & nothing\\
\hline
\textbf{Seconds} & $17$  & $2.6$ & $275$ & $212 \times 10^3$ & $12$ & $5.3$\\
\hline
\end{tabular}
\end{center}
\label{tbl:results}
\end{table*}

In this section we evaluate a variant of \algref{erucb} against
several baselines. While \algref{erucb} is significantly more
efficient than many previous approaches, the overall computational
complexity is still at least $\otil((KT)^{1.5})$ plus the total cost of
the oracle calls, as discussed in Section~\ref{sec:complexity}.
This is markedly larger than the
complexity of an ordinary supervised learning problem where it is
typically possible to perform an $O(1)$-complexity update upon
receiving a fresh example using online algorithms.

A natural solution is to use an \emph{online} oracle that is stateful
and accepts examples one by one.  An online cost-sensitive classification
(CSC) oracle takes as input
a weighted example and returns a predicted class (corresponding to one
of $K$ actions in our setting). Since the oracle is stateful, it
remembers and uses examples from all previous calls in answering
questions, thereby reducing the complexity of each oracle invocation
to $O(1)$ as in supervised learning. Using several such oracles,
we can efficiently track a distribution over good policies and sample
from it. We detail this approach
(which we call Online Cover) in the full version of the paper.
The algorithm maintains a uniform distribution
over a fixed number $n$ of policies where $n$ is a parameter of the
algorithm. Upon receiving a fresh example, it updates all $n$ policies
with the suitable CSC examples (\eqref{cse}). The specific CSC
oracle we use is a reduction to squared-loss regression
(Algorithms 4 and 5 of~\citet{Beygelzimer09offset}) which is amenable
to online updates. Our implementation is included in Vowpal
Wabbit.\footnote{\url{http://hunch.net/~vw}. The implementation is in
  the file \texttt{cbify.cc} and is enabled using \texttt{--cover}.}

Due to lack of public datasets for contextual bandit problems, we use
a simple supervised-to-contextual-bandit
transformation~\citep{DoubleRobust} on the CCAT document classification
problem in RCV1~\citep{lewis2004rcv1}.  This dataset has $781265$
examples and $47152$ TF-IDF features. We treated the class labels as actions, and one minus
0/1-loss as the reward.
Our evaluation criteria is progressive
validation~\citep{ProgressiveValidation} on 0/1 loss.
We compare several baseline algorithms to Online Cover; all algorithms
take advantage of linear representations which are known to work well
on this dataset. For each algorithm, we report the result for the best
parameter settings (shown in Table~\ref{tbl:results}).
\begin{enumerate}
\item $\epsilon$-greedy~\citep{sutton-barto} explores randomly with
  probability $\epsilon$ and otherwise exploits. 
\item Explore-first is a variant that begins with uniform
  exploration, then switches to an exploit-only phase.
\item A less common but powerful baseline is based on bagging:
  multiple predictors (policies) are trained with examples sampled
  with replacement.
  Given a context, these predictors yield
  a distribution over actions from which we
  can sample.
\item LinUCB~\citep{Linear,LinUCB} has been quite effective
  in past evaluations~\citep{Li10Contextual,ThompsonExp}.
  It is impractical to run ``as is'' due to
  high-dimensional matrix inversions, so
  we
  report results for this algorithm after reducing to
  $1000$ dimensions via random
  projections. Still, the
  algorithm required $59$ hours\footnote{The linear algebra routines are
  based on Intel MKL package.}. An alternative is to use diagonal
  approximation to the covariance, which runs substantially faster
  ($\approx$1 hour), but gives a worse error of 0.137.
\item Finally, our algorithm achieves the best loss of $0.0530$.
  Somewhat surprisingly, the minimum occurs for us with a cover set of
  size 1---apparently for this problem the small decaying amount of
  uniform random sampling imposed is adequate exploration.  Prediction
  performance is similar with a larger cover set.
\end{enumerate}

All baselines except for LinUCB are implemented as a simple
modification of Vowpal Wabbit.
All reported results use default parameters where not otherwise
specified.  The contextual bandit learning algorithms all use a doubly
robust reward estimator instead of the importance weighted estimators
used in our analysis~\cite{DoubleRobust}.

Because RCV1 is actually a fully supervised dataset, we can apply a
fully supervised online multiclass algorithm to solve it.  We use a
simple one-against-all implementation to reduce this to binary
classification, yielding an error rate of $0.051$ which is competitive
with the best previously reported results.  This is effectively
a lower bound on the loss we can hope to achieve
with algorithms using only partial information.
Our algorithm is less than 2.3 times slower and nearly achieves the bound.
Hence on this dataset, very little further algorithmic
improvement is possible.

\section{Conclusions}
\label{sec:conclusions}

In this paper we have presented the first practical algorithm to our
knowledge that attains the statistically optimal regret guarantee and
is computationally efficient in the setting of general policy
classes. A remarkable feature of the algorithm is that the total
number of oracle calls over all $T$ rounds is sublinear---a remarkable
improvement over previous works in this setting. We believe that the
online variant of the approach which we implemented in our experiments
has the right practical flavor for a scalable solution to the
contextual bandit problem. In future work, it would be interesting to
directly analyze the Online Cover algorithm.

\subsection*{Acknowledgements}
We thank Dean Foster and Matus Telgarsky for helpful discussions.
Part of this work was completed while DH and RES were visiting Microsoft
Research.

\bibliography{../minimonster}
\bibliographystyle{plainnat}

\appendix

\section{Omitted Algorithm Details}
\label{app:details}

\algref{IPS} and \algref{SAMPLE} give the details of the inverse propensity
scoring transformation $\IPS$ and the action sampling procedure $\SAMPLE$.

\begin{algorithm}[t]
  \caption{$\IPS(H)$}
  \label{alg:IPS}
  \begin{algorithmic}[1]
    \renewcommand{\algorithmicrequire}{\textbf{input}}
    \renewcommand{\algorithmicensure}{\textbf{output}}

    \REQUIRE History $H \subseteq X \times A \times [0,1] \times [0,1]$.

    \ENSURE Data set $S \subseteq X \times \R_+^A$.

    \STATE Initialize data set $S := \emptyset$.

    \FOR{\textbf{each} $(x,a,r(a),p(a)) \in H$}

      \STATE Create fictitious rewards $\hat{r} \in \R_+^A$ with
      $\hat{r}(a) = r(a)/p(a)$ and $\hat{r}(a') = 0$ for all $a' \in A
      \setminus \{ a \}$.

      \STATE $S := S \cup \{ (x,\hat{r}) \}$.

    \ENDFOR

    \RETURN $S$.

  \end{algorithmic}
\end{algorithm}

\begin{algorithm}[t]
  \caption{$\SAMPLE(x,Q,\bar\pi,\mu)$}
  \label{alg:SAMPLE}
  \begin{algorithmic}[1]
    \renewcommand{\algorithmicrequire}{\textbf{input}}
    \renewcommand{\algorithmicensure}{\textbf{output}}

    \REQUIRE Context $x \in X$, weights $Q \in \Delta^\Pi$, default policy
    $\bar\pi \in \Pi$, minimum probability $\mu \in [0,1/K]$.

    \ENSURE Selected action $\bar{a} \in A$ and probability $\bar{p} \in
    [\mu,1]$.

    \STATE Let $\tilde{Q} := Q + (1 - \sum_{\pi \in \Pi} Q(\pi))
    \one_{\bar\pi}$ \\
    (so $\sum_{\pi \in \Pi} \tilde{Q}(\pi) = 1$).
    \label{step:make-distribution}

    \STATE Randomly draw action $\bar{a} \in A$ using the distribution
    \[
      \tilde{Q}^\mu(a|x) := (1 - K\mu) \sum_{\substack{\pi \in \Pi
      : \\ \pi(x) = a}}
      \tilde{Q}(\pi) + \mu , \quad \forall a \in A .
    \]

    \STATE Let $\bar{p}(\bar{a}) := \tilde{Q}^\mu(\bar{a}|x)$.

    \RETURN $(\bar{a},\bar{p}(\bar{a}))$.

  \end{algorithmic}
\end{algorithm}

\section{Deviation Inequalities}
\label{app:deviation}

\subsection{Freedman's Inequality}

The following form of Freedman's inequality for martingales is
from~\citet{Exp4p}.
\begin{lemma}
Let $X_1, X_2, \dotsc, X_T$ be a sequence of real-valued random variables.
Assume for all $t \in \{1,2,\dotsc,T\}$, $X_t \leq R$ and
$\E[X_t|X_1,\dotsc,X_{t-1}] = 0$.
Define $S := \sum_{t=1}^T X_t$ and $V := \sum_{t=1}^T
\E[X_t^2|X_1,\dotsc,X_{t-1}]$.
For any $\delta \in (0,1)$
and $\lambda \in [0,1/R]$,
with probability at least $1-\delta$,
\begin{equation*}
  S \leq (e-2)\lambda V + \frac{\ln(1/\delta)}{\lambda}
  .
\end{equation*}
\end{lemma}

\subsection{Variance Bounds}

Fix the epoch schedule $0 = \tau_0 < \tau_1 < \tau_2 < \dotsb$.

Define the following for any probability distribution $P$ over $\Pi$, $\pi
\in \Pi$, and $\mu \in [0,1/K]$:
\begin{align}
  V(P,\pi,\mu)
  & := \E_{x \sim \D_X}\left[ \frac{1}{P^\mu(\pi(x)|x)} \right] ,
  \label{eq:V}
  \\
  \wh{V}_m(P,\pi,\mu)
  & := \wh\E_{x \sim H_{\tau_m}}
  \left[ \frac{1}{P^\mu(\pi(x)|x)} \right] .
  \label{eq:Vhat}
\end{align}

The proof of the following lemma is essentially the same as that of Theorem
6 from \citet{Monster}.
\begin{lemma}
\label{lem:variance-bounds}
Fix any $\mu_m \in [0,1/K]$ for $m \in \N$.
For any $\delta \in (0,1)$, with probability at least $1-\delta$,
\begin{equation*}
  V(P,\pi,\mu_m)
  \leq 6.4 \wh{V}_m(P,\pi,\mu_m) 
  + \frac{75(1-K\mu_m)\ln|\Pi|}{\mu_m^2 \tau_m}
  + \frac{6.3\ln(2|\Pi|^2m^2/\delta)}{\mu_m \tau_m}
\end{equation*}
for all probability distributions $P$ over $\Pi$, all $\pi \in \Pi$, and
all $m \in \N$.
In particular, if
\begin{align*}
  \mu_m & \geq \sqrt{\frac{\ln(2|\Pi|m^2/\delta)}{K\tau_m}} , &
  \tau_m & \geq 4K\ln(2|\Pi|m^2/\delta) ,
\end{align*} 
then
\begin{equation*}
  V(P,\pi,\mu_m) \leq 6.4 \wh{V}_m(P,\pi,\mu_m) + 81.3 K .
\end{equation*}
\end{lemma}
\begin{proof}[Proof sketch]
By Bernstein's (or Freedman's) inequality and union bounds, for any choice
of $N_m \in \N$
and
$\lambda_m \in [0,\mu_m]$
for $m \in \N$,
the following holds with probability at
least $1-\delta$:
\begin{equation*}
V(P,\pi,\mu_m)
- \wh{V}_m(P,\pi,\mu_m)
\leq \frac{(e-2) \lambda_m V(P,\pi,\mu_m)}{\mu_m}
+ \frac{\ln(|\Pi|^{N_m+1}2m^2/\delta)}{\lambda_m \tau_m}
\end{equation*}
all $N_m$-point distributions $P$ over $\Pi$,
all $\pi \in \Pi$, and
all $m \in \N$.
Here, an $N$-point distribution over $\Pi$ is a distribution of the form
$\frac1N \sum_{i=1}^N \one_{\pi_i}$ for $\pi_1, \pi_2, \dotsc, \pi_N \in
\Pi$.
We henceforth condition on this $\geq 1-\delta$ probability event (for
choices of $N_m$ and $\lambda_m$ to be determined).

Using the probabilistic method (for more details, we refer the reader
to the proof of Theorem 6 from \citet{Monster}), it can be shown that
for any probability distribution $P$ over $\Pi$, any $\pi \in \Pi$,
any $\mu_m \in [0,1/K]$, and any $c_m>0$, there exists an $N_m$-point
distribution $\wt{P}$ over $\Pi$ such that
\begin{multline*}
  \bigl( V(P,\pi,\mu_m) - V(\wt{P},\pi,\mu_m) \bigr)
  + c_m\bigl( \wh{V}_m(\wt{P},\pi,\mu_m) - \wh{V}_m(P,\pi,\mu_m) \bigr)
  \\
  \leq \gamma_{N_m,\mu_m}
  \bigl( V(P,\pi,\mu_m) + c_m \wh{V}_m(P,\pi,\mu_m) \bigr)
\end{multline*}
where $\gamma_{N,\mu} := \sqrt{(1-K\mu)/(N\mu)} + 3(1-K\mu)/(N\mu)$.

Combining the displayed inequalities (using $c_m :=
1/(1-(e-2)\lambda_m/\mu_m)$) and rearranging gives
\begin{equation*}
  V(P,\pi,\mu_m)
  \leq \frac{1+\gamma_{N_m,\mu_m}}{1-\gamma_{N_m,\mu_m}}
  \cdot \frac{\wh{V}_m(P,\pi,\mu_m)}{1-(e-2)\frac{\lambda_m}{\mu_m}}
  + \frac{1}{1-\gamma_{N_m,\mu_m}}
  \cdot \frac{1}{1-(e-2)\frac{\lambda_m}{\mu_m}}
  \cdot \frac{\ln(|\Pi|^{N_m+1} 2m^2/\delta)}{\lambda_m \tau_m}
  .
\end{equation*}
Using $N_m := \lceil 12(1-K\mu_m)/\mu_m \rceil$ and
$\lambda_m := 0.66\mu_m$ for all $m \in \N$
gives the claimed inequalities.

If
$\mu_m \geq \sqrt{\ln(2|\Pi|m^2/\delta)/(K\tau_m)}$ and $\tau_m \geq
4K\ln(2|\Pi|m^2/\delta)$, then $\mu_m^2\tau_m \geq \ln(|\Pi|)/K$
and $\mu_m\tau_m \geq \ln(2|\Pi|^2m^2/\delta)$, and hence
\begin{equation*}
  \frac{75(1-K\mu_m)\ln|\Pi|}{\mu_m^2\tau_m}
  + \frac{6.3\ln(2|\Pi|^2m^2/\delta)}{\mu_m \tau_m}
  \leq (75 + 6.3)K = 81.3 K.
  \qedhere
\end{equation*}
\end{proof}

\subsection{Reward Estimates}

Again, fix the epoch schedule $0 = \tau_0 < \tau_1 < \tau_2 < \dotsb$.
Recall that for any epoch $m \in \N$ and round $t$ in epoch $m$,
\begin{itemize}
  \item $Q_{m-1} \in \Delta^\Pi$ are the non-negative weights computed at
    the end of epoch $m-1$;

  \item $\wt{Q}_{m-1}$ is the probability distribution over $\Pi$ obtained
    from $Q_{m-1}$ and the policy $\pi_{m-1}$ with the highest reward
    estimate through epoch $m-1$;

  \item $\wt{Q}_{m-1}^{\mu_{m-1}}(\cdot|x_t)$ is the probability
    distribution used to choose $a_t$.

\end{itemize}

Let
\begin{equation} \label{eq:epoch-start}
m(t) := \min\{ m \in \N : t \leq \tau_m \}
\end{equation}
be the index of the epoch containing round $t \in \N$, and
define
\begin{equation}
\label{eq:Vmax}
\Vmax{t}(\pi)
:= \max_{0 \leq m \leq m(t)-1}
\{ V(\wt{Q}_m,\pi,\mu_m) \}
\end{equation}
for all $t \in \N$ and $\pi \in \Pi$.

\begin{lemma}
\label{lem:reward-estimates}
For any $\delta \in (0,1)$
and any choices of $\lambda_{m-1} \in [0,\mu_{m-1}]$ for $m \in \N$,
with probability at least $1-\delta$,
\begin{equation*}
|\wh\Re_t(\pi) - \Re(\pi)|
\leq \Vmax{t}(\pi) \lambda_{m-1}
+ \frac{\ln(4t^2|\Pi|/\delta)}{t\lambda_{m-1}}
\end{equation*}
for
all policies $\pi \in \Pi$,
all epochs $m \in \N$,
and all rounds $t$ in epoch $m$.
\end{lemma}
\begin{proof}
Fix any policy $\pi \in \Pi$, epoch $m \in \N$, and round $t$ in epoch $m$.
Then
\begin{equation*}
\wh\Re_t(\pi) - \Re(\pi) = \frac1t \sum_{i=1}^t Z_i
\end{equation*}
where $Z_i := \hat{r}_i(\pi(x_i)) - r_i(\pi(x_i))$.
Round $i$ is in epoch $m(i) \leq m$, so
\begin{equation*}
|Z_i|
\leq \frac{1}{\wt{Q}_{m(i)-1}^{\mu_{m(i)-1}}(\pi(x_i)|x_i)}
\leq \frac{1}{\mu_{m(i)-1}}
\end{equation*}
by the definition of the fictitious rewards.
Because the sequences $\mu_1 \geq \mu_2 \geq \dotsb$ and $m(1) \leq
m(2) \leq \dotsb$ are monotone, it follows that $Z_i \leq
1/\mu_{m-1}$ for all $1 \leq i \leq t$.
Furthermore,
$\E[ Z_i | H_{i-1} ] = 0$
and
\begin{align*}
\E[ Z_i^2 | H_{i-1} ]
& \leq \E[ \hat{r}_i(\pi(x_i))^2 | H_{i-1} ] \\
& \leq V(\wt{Q}_{m(i)-1}, \pi, \mu_{m(i)-1})
\leq \Vmax{t}(\pi)
\end{align*}
for all $1 \leq i \leq t$.
The first inequality follows because for $\var(X) \leq \E(X^2)$ for any
random variable $X$; and the other inequalities follow from the definitions
of the fictitious rewards, $V(\cdot,\cdot,\cdot)$ in \eqref{V}, and
$\Vmax{t}(\cdot)$ in \eqref{Vmax}.
Applying Freedman's inequality and a union bound to the sums $(1/t)
\sum_{i=1}^t Z_i$ and $(1/t) \sum_{i=1}^t (-Z_i)$ implies the following:
for all $\lambda_{m-1} \in [0,\mu_{m-1}]$,
with probability at least $1-2 \cdot \delta/(4t^2|\Pi|)$,
\begin{equation*}
\left| \frac1t \sum_{i=1}^t Z_i \right|
\leq (e-2) \Vmax{t}(\pi) \lambda_{m-1}
+ \frac{\ln(4t^2|\Pi|/\delta)}{t\lambda_{m-1}}
.
\end{equation*}
The lemma now follows by applying a union bound for all choices of $\pi \in
\Pi$ and $t \in \N$, since
\begin{equation*}
  \sum_{\pi \in \Pi} \sum_{t \in \N} \frac{\delta}{2t^2|\Pi|} \leq \delta .
\end{equation*}
\end{proof}

\section{Regret Analysis}
\label{app:regret}

Throughout this section, we fix the allowed probability of failure $\delta
\in (0,1)$ provided as input to the algorithm, as well as the epoch
schedule $0 = \tau_0 < \tau_1 < \tau_2 < \dotsb$.

\subsection{Definitions}

Define, for all $t \in \N$,
\begin{align}
  \label{eq:dt-def}
  d_t & := \ln(16t^2|\Pi|/\delta) ,
\end{align}
and recall that,
\begin{align*}
  \mu_m & = \min\left\{ \frac{1}{2K} , \, \sqrt{\frac{d_{\tau_m}}{K\tau_m}} \right\} .
\end{align*}
Observe that
$d_t/t$ is non-increasing with $t \in \N$, and
$\mu_m$ is non-increasing with $m \in \N$.

Let
\begin{align*}
  m_0 & := \min \left\{
                  m \in \N : \frac{d_{\tau_m}}{\tau_m} \leq \frac1{4K}
                \right\}
                .
\end{align*}
Observe that $\tau_{m_0} \geq 2$.

Define
\begin{equation*}
  \ratio := \sup_{m \geq m_0}
  \left\{ \sqrt{\frac{\tau_m}{\tau_{m-1}}} \right\} .
\end{equation*}
Recall that we assume $\tau_{m+1} \leq 2\tau_m$; thus $\ratio \leq \sqrt2$.

\subsection{Deviation Control and Optimization Constraints}

Let $\GoodEvent$ be the event in which the following statements hold:
\begin{equation}
  \label{eq:event1}
  V(P,\pi,\mu_m)
  \leq 6.4 \wh{V}_m(P,\pi,\mu_m)
  + 81.3 K
\end{equation}
for all probability distributions $P$ over $\Pi$, all $\pi \in \Pi$, and
all $m \in \N$ such that
$\tau_m \geq 4Kd_{\tau_m}$
(so $\mu_m = \sqrt{d_{\tau_m} / (K\tau_m)}$); and
\begin{equation}
  \label{eq:event2}
  |\wh\Re_t(\pi) - \Re(\pi)|
  \leq
  \begin{cases}
    \displaystyle
    \max\braces{ \sqrt{\frac{3\Vmax{t}d_t}{t}} ,\, \frac{2\Vmax{t}d_t}{t} }
    & \text{if $m \leq m_0$} , \\
    \displaystyle
    \Vmax{t}(\pi) \mu_{m-1} + \frac{d_t}{t\mu_{m-1}}
    & \text{if $m > m_0$} .
  \end{cases}
\end{equation}
for all
all policies $\pi \in \Pi$,
all epochs $m \in \N$,
and all rounds $t$ in epoch $m$.
By \lemref{variance-bounds}, \lemref{reward-estimates}, and a union
bound, $\Pr(\GoodEvent) \geq 1-\delta/2$.

For every epoch $m \in \N$, the weights $Q_m$ computed at the end of the
epoch (in round $\tau_m$) as the solution to (\optprob) satisfy the
constraints \eqref{reg-cons} and \eqref{var-cons}: they are, respectively:
\begin{equation}
  \label{eq:opt-constraint1}
  \sum_{\pi \in \Pi} Q_m(\pi) \wh\Reg_{\tau_m}(\pi)
  \leq \vlc \cdot 2K\mu_m
\end{equation}
and, for all $\pi \in \Pi$,
\begin{equation}
  \label{eq:opt-constraint2}
  \wh{V}_m(Q_m,\pi,\mu_m)
  \leq 2 K + \frac{\wh\Reg_{\tau_m}(\pi)}{\vlc \cdot \mu_m}
  .
\end{equation}
Recall that $\vlc = \vlcvalue$ (as defined in (\optprob), assuming
$\ratio \leq \sqrt2$).  Define $\theta_1 := 94.1$ and $\theta_2 :=
\vlc/6.4$ (needed for the next \lemref{vmax}). With these settings,
the proof of \lemref{inductive} will require that $\theta_2 \geq
8\ratio$, and hence $\vlc \geq 6.4 \cdot 8 \ratio$; this is true with
our setting of $\vlc$ since $\ratio \leq \sqrt2$.

\subsection{Proof of \thmref{regret-main}}

We now give the proof of \thmref{regret-main}, following the outline in
\secref{regret}.

The following lemma shows that if $\Vmax{t}(\pi)$ is large---specifically,
much larger than $K$---then the estimated regret of $\pi$ was large in some
previous round.

\begin{lemma}
\label{lem:vmax}
Assume event $\GoodEvent$ holds.
Pick any round $t \in \N$ and any policy $\pi \in \Pi$, and let $m \in \N$
be the epoch achieving the $\max$ in the definition of $\Vmax{t}(\pi)$.
Then
\begin{equation*}
  \Vmax{t}(\pi)
  \leq
  \begin{cases}
    2K & \quad \text{if $\mu_m = 1/(2K)$} , \\
  \theta_1 K
  + \displaystyle\frac{\wh\Reg_{\tau_m}(\pi)}{\theta_2 \mu_m}
  & \quad \text{if $\mu_m < 1/(2K)$} .
  \end{cases}
\end{equation*}
\end{lemma}
\begin{proof}
Fix a round $t \in \N$ and policy $\pi \in \Pi$.
Let ${m} \leq m(t)-1$ be the epoch achieving the $\max$ in the definition of
$\Vmax{t}(\pi)$ from \eqref{Vmax}, so $\Vmax{t}(\pi) =
V(\wt{Q}_m,\pi,\mu_m)$.
If $\mu_m = 1/(2K)$, then $V(\wt{Q}_m,\pi,\mu_m) \leq 2K$.
So assume instead that $1/(2K) > \mu_m = \sqrt{d_{\tau_m}/(K\tau_m)}$.
This implies that $\tau_m > 4Kd_{\tau_m}$.
By \eqref{event1}, which holds in event $\GoodEvent$,
\begin{equation*}
  V(\wt{Q}_m,\pi,\mu_m)
  \leq 6.4\wh{V}_m(\wt{Q}_m,\pi,\mu_m) + 81.3 K .
\end{equation*}
The probability distribution $\wt{Q}_m$ satisfies the inequalities
\begin{equation*}
  \wh{V}_m(\wt{Q}_m,\pi,\mu_m)
  \leq \wh{V}_m(Q_m,\pi,\mu_m)
  \leq 2K + \frac{\wh\Reg_{\tau_m}(\pi)}{\vlc \mu_m}
  .
\end{equation*}
Above, the first inequality follows because the value of
$\wh{V}_m(Q_m,\pi,\mu_m)$ decreases as the value of
$Q_m(\pi_{\tau_m})$ increases, as it does when going from $Q_m$ to
$\wt{Q}_m$;
the second inequality is the constraint \eqref{opt-constraint2} satisfied
by $Q_m$.
Combining the displayed inequalities from above proves the claim.
\end{proof}

In the next lemma, we compare $\Reg(\pi)$ and $\wh\Reg_t(\pi)$ for any
policy $\pi$ by using the deviation bounds for estimated rewards together
with the variance bounds from \lemref{vmax}.
Define $t_0 := \min \{ t \in \N : d_t/t \leq 1/(4K) \}$.

\begin{lemma}
\label{lem:inductive}
Assume event $\GoodEvent$ holds.
Let $\bigc := 4\ratio(1+\theta_1)$.
For all epochs $m \geq m_0$,
all rounds $t \geq t_0$ in epoch $m$,
and all policies $\pi \in \Pi$,
\begin{align*}
  \Reg(\pi)
  & \leq 2 \wh\Reg_t(\pi)
  + \bigc K\mu_m
  ;
  \\
  \wh\Reg_t(\pi)
  & \leq 2 \Reg(\pi)
  + \bigc K\mu_m
  .
\end{align*}
\end{lemma}
\begin{proof}
The proof is by induction on $m$.
As the base case, consider $m = m_0$ and $t \geq t_0$ in epoch $m$.
By definition of $m_0$, $\mu_{m'} = 1/(2K)$ for all $m' < m_0$, so
$\Vmax{t}(\pi) \leq 2K$ for all $\pi \in \Pi$ by \lemref{vmax}.
By \eqref{event2}, which holds in event $\GoodEvent$, for all $\pi \in
\Pi$,
\begin{align*}
|\wh\Re_t(\pi) - \Re(\pi)|
& \leq \max\braces{ \sqrt{\frac{6Kd_t}{t}} ,\, \frac{4Kd_t}{t} }
\leq \sqrt{\frac{6Kd_t}{t}}
\end{align*}
where we use the fact that $4Kd_t/t \leq 1$ for $t \geq t_0$.
This implies
\begin{align*}
|\wh\Reg_t(\pi) - \Reg(\pi)|
& \leq 
2\sqrt{\frac{6Kd_t}{t}} .
\end{align*}
by the
triangle inequality and optimality of $\pi_t$ and $\piopt$.
Since $t > \tau_{m_0-1}$ and
$\bigc \geq 2\sqrt6\ratio$, it follows that
$|\wh\Reg_t(\pi) - \Reg(\pi)| \leq
2\sqrt{6}\ratio K\mu_{m_0} \leq \bigc K\mu_{m_0}$.

For the inductive step, fix some epoch $m > m_0$.
We assume as the inductive hypothesis that for all epochs $m' < m$, all
rounds $t'$ in epoch $m'$, and all $\pi \in \Pi$,
\begin{align*}
  \Reg(\pi)
  & \leq 2 \wh\Reg_{t'}(\pi)
  + \bigc K\mu_{m'}
  ;
  \\
  \wh\Reg_{t'}(\pi)
  & \leq 2 \Reg(\pi)
  + \bigc K\mu_{m'}
  .
\end{align*}
We first show that
\begin{equation}
  \label{eq:i1}
  \Reg(\pi)
  \leq 2 \wh\Reg_t(\pi)
  + \bigc K\mu_m
\end{equation}
for all rounds $t$ in epoch $m$ and all $\pi \in \Pi$.
So fix such a round $t$ and policy $\pi$;
by \eqref{event2} (which holds in event $\GoodEvent$),
\begin{align}
  \Reg(\pi)
  - 
  \wh\Reg_t(\pi)
  & = \bigl( \Re(\piopt) - \Re(\pi) \bigr)
  - 
  \bigl( \wh\Re_t(\pi_t) - \wh\Re_t(\pi) \bigr)
  \nonumber \\
  & \leq \bigl( \Re(\piopt) - \Re(\pi) \bigr)
  - 
  \bigl( \wh\Re_t(\piopt) - \wh\Re_t(\pi) \bigr)
  \nonumber \\
  & \leq
  \bigl( \Vmax{t}(\pi) + \Vmax{t}(\piopt) \bigr) \mu_{m-1} +
  \frac{2d_t}{t\mu_{m-1}}
  .
  \label{eq:i1-1}
\end{align}
Above, the first inequality follows from the optimality of $\pi_t$.
By \lemref{vmax}, there exist epochs $i,j < m$ such that
\begin{align*}
  \Vmax{t}(\pi)
  & \leq \theta_1 K
  + \frac{\wh\Reg_{\tau_i}(\pi)}{\theta_2 \mu_i}
  \cdot \ind{\mu_i < 1/(2K)}
  ,
  \\
  \Vmax{t}(\piopt)
  & \leq \theta_1 K
  + \frac{\wh\Reg_{\tau_j}(\piopt)}{\theta_2 \mu_j}
  \cdot \ind{\mu_j < 1/(2K)}
  .
\end{align*}
Suppose $\mu_i < 1/(2K)$, so $m_0 \leq i < m$: in this case, the inductive
hypothesis implies
\begin{align*}
  \frac{\wh\Reg_{\tau_i}(\pi)}{\theta_2 \mu_i}
  & \leq \frac{2 \Reg(\pi) + \bigc K \mu_i}
  {\theta_2 \mu_i}
  \leq \frac{\bigc K}{\theta_2}
  + \frac{2\Reg(\pi)}{\theta_2 \mu_{m-1}}
\end{align*}
where the second inequality uses the fact that $i \leq m-1$.
Therefore,
\begin{align}
  \Vmax{t}(\pi) \mu_{m-1}
  & \leq \left(\theta_1 + \frac{\bigc}{\theta_2}\right) K \mu_{m-1}
  + \frac{2}{\theta_2} \Reg(\pi)
  .
  \label{eq:i1-2}
\end{align}
Now suppose $\mu_j < 1/(2K)$, so $m_0 \leq j < m$: as above, the inductive
hypothesis implies
\begin{align*}
  \frac{\wh\Reg_{\tau_j}(\piopt)}{\theta_2 \mu_j}
  & \leq \frac{2 \Reg(\piopt) + \bigc K \mu_j}
  {\theta_2 \mu_j}
  = \frac{\bigc}{\theta_2} K
\end{align*}
since $\Reg(\piopt) = 0$.
Therefore,
\begin{align}
  \Vmax{t}(\piopt) \mu_{m-1}
  & \leq \left(\theta_1 + \frac{\bigc}{\theta_2}\right) K \mu_{m-1}
  .
  \label{eq:i1-3}
\end{align}
Combining \eqref{i1-1}, \eqref{i1-2}, and \eqref{i1-3}, and rearranging
gives
\begin{equation*}
\Reg(\pi)
\leq
\frac{1}{1 - \frac{2}{\theta_2}}
\biggl(
  \wh\Reg_t(\pi)
  + 2\biggl( \theta_1 + \frac{\bigc}{\theta_2} \biggr) K\mu_{m-1}
  + \frac{2d_t}{t\mu_{m-1}}
\biggr)
.
\end{equation*}
Since $m \geq m_0+1$, it follows that $\mu_{m-1} \leq \ratio \mu_m$ by
definition of $\ratio$.
Moreover, since $t > \tau_{m-1}$, $(d_t/t) / \mu_{m-1} \leq K\mu_{m-1}^2 /
\mu_{m-1} \leq \ratio K\mu_m$
Applying these inequalities to the above display, and simplifying, yields
\eqref{i1} because $\bigc \geq 4\ratio(1+\theta_1)$ and $\theta_2 \geq
8\ratio$.

We now show that
\begin{equation}
  \label{eq:i2}
  \wh\Reg_t(\pi)
  \leq 2 \Reg(\pi) + \bigc K\mu_m
\end{equation}
for all $\pi \in \Pi$.
Again, fix an arbitrary $\pi \in \Pi$, and by~\eqref{event2},
\begin{align}
  \wh\Reg_t(\pi) - \Reg(\pi)
  & =
  \bigl( \wh\Re_t(\pi_t) - \wh\Re_t(\pi) \bigr)
  - 
  \bigl( \Re(\piopt) - \Re(\pi) \bigr)
  \nonumber \\
  & \leq
  \bigl( \wh\Re_t(\pi_t) - \wh\Re_t(\pi) \bigr)
  - 
  \bigl( \Re(\pi_t) - \Re(\pi) \bigr)
  \nonumber \\
  & \leq
  \bigl( \Vmax{t}(\pi) + \Vmax{t}(\pi_t) \bigr) \mu_{m-1} +
  \frac{2d_t}{t\mu_{m-1}}
  \label{eq:i2-1}
\end{align}
where the first inequality follows from the optimality of $\piopt$.
By \lemref{vmax}, there exists an epoch $j < m$ such
\begin{align*}
  \Vmax{t}(\pi_t)
  & \leq \theta_1 K
  + \frac{\wh\Reg_{\tau_j}(\pi_t)}{\theta_2 \mu_j}
  \cdot \ind{\mu_j < 1/(2K)}
  .
\end{align*}
Suppose $\mu_j < 1/(2K)$, so $m_0 \leq j < m$: in this case
the inductive hypothesis and \eqref{i1} imply
\begin{equation*}
  \frac{\wh\Reg_{\tau_j}(\pi_t)}{\theta_2 \mu_j}
  \leq 
  \frac{2 \Reg(\pi_t) + \bigc K \mu_j}
  {\theta_2 \mu_j}
  \leq 
  \frac{2 \Bigl( 2 \wh\Reg_t(\pi_t) + \bigc K\mu_m \Bigr)
  + \bigc K \mu_j}
  {\theta_2 \mu_j}
  =
  \frac{3\bigc}{\theta_2} K
\end{equation*}
(the last equality follows because $\wh\Reg_t(\pi_t) = 0$).
Thus
\begin{equation}
  \label{eq:i2-2}
  \Vmax{t}(\pi_t) \mu_{\tau(t)-1}
  \leq \left(\theta_1 + \frac{3\bigc}{\theta_2}\right)
  K \mu_{m-1}
  .
\end{equation}
Combining \eqref{i2-1}, \eqref{i2-2}, and \eqref{i1-2} gives
\begin{equation*}
  \wh\Reg_t(\pi) \leq \left( 1 + \frac{2}{\theta_2} \right) \Reg(\pi)
  + \left( 2\theta_1 + \frac{4\bigc}{\theta_2} \right)
  K\mu_{m-1}
  + \frac{2d_t}{t\mu_{m-1}}
  .
\end{equation*}
Again, applying the inequalities $\mu_{m-1} \leq \ratio\mu_m$ and $(d_t/t)
/ \mu_{m-1} \leq K\mu_m$ to the above display, and simplifying, yields
\eqref{i2} because $\bigc \geq 4\ratio(1+\theta_1)$ and $\theta_2 \geq
8\ratio$.
This completes the inductive step, and thus proves the overall claim.
\end{proof}

The next lemma shows that the ``low estimated regret guarantee'' of
$Q_{t-1}$ (optimization constraint \eqref{opt-constraint1}) also implies a ``low
regret guarantee'', via the comparison of $\wh\Reg_t(\cdot)$ to
$\Reg(\cdot)$ from \lemref{inductive}.

\begin{lemma}
\label{lem:expected-regret}
Assume event $\GoodEvent$ holds.
For every epoch $m \in \N$,
\begin{equation*}
  \sum_{\pi \in \Pi} \wt{Q}_{m-1}(\pi) \Reg(\pi)
  \leq (4\vlc + \bigc) K \mu_{m-1}
\end{equation*}
where $\bigc$ is defined in \lemref{inductive}.
\end{lemma}
\begin{proof}
Fix any epoch $m \in \N$.
If $m \leq m_0$, then $\mu_{m-1} = 1/(2K)$, in which case the claim is
trivial.
Therefore assume $m \geq m_0+1$.
Then
\begin{align*}
    \sum_{\pi \in \Pi} \wt{Q}_{m-1}(\pi) \Reg(\pi)
  & \leq \sum_{\pi \in \Pi} \wt{Q}_{m-1}(\pi)
  \bigl( 2 \wh\Reg_{\tau_{m-1}}(\pi) + \bigc K\mu_{m-1} \bigr)
  \\
  & = \biggl(
  2 \sum_{\pi \in \Pi} Q_{m-1}(\pi) \wh\Reg_{\tau_{m-1}}(\pi)
  \biggr)
  + \bigc K\mu_{m-1}
  \\
  & \leq \vlc \cdot 4 K\mu_{m-1}
  + \bigc K\mu_{m-1}
  .
\end{align*}
The first step follows from \lemref{inductive}, as all rounds in an epoch
$m \geq m_0+1$ satisfy $t \geq t_0$;
the second step follows from the fact that $\wt{Q}_{m-1}$ is a
probability distribution, that $\wt{Q}_{m-1} = Q_{m-1} +
\alpha \one_{\pi_{\tau_{m-1}}}$ for some $\alpha \geq 0$, and that
$\wh\Reg_{\tau_{m-1}}(\pi_{\tau_{m-1}}) = 0$;
and the last step follows from the constraint \eqref{opt-constraint1}
satisfied by $Q_{m-1}$.
\end{proof}

Finally, we straightforwardly translate the ``low regret guarantee'' from
\lemref{expected-regret} to a bound on the cumulative regret of the
algorithm.
This involves summing the bound in \lemref{expected-regret} over all rounds
$t$ (\lemref{sum-mu} and \lemref{sum-mu'}) and applying a martingale
concentration argument (\lemref{regret}).

\begin{lemma}
\label{lem:sum-mu}
For any $T \in \N$,
\begin{equation*}
  \sum_{t=1}^T \mu_{m(t)} \leq
  2\sqrt{\frac{d_{\tau_{m(T)}}\tau_{m(T)}}{K}}
  .
\end{equation*}
\end{lemma}
\begin{proof}
  We break the sum over rounds into the epochs, and bound the sum within
  each epoch:
  \begin{align*}
    \sum_{t=1}^T \mu_{m(t)}
    & \leq
    \sum_{m=1}^{m(T)}
    \sum_{t=\tau_{m-1}+1}^{\tau_m} \mu_m
    \\
    & \leq \sum_{m=1}^{m(T)}
    \sum_{t=\tau_{m-1}+1}^{\tau_m}
    \sqrt{\frac{d_{\tau_m}}{K\tau_m}}
    \\
    & \leq
    \sqrt{\frac{d_{\tau_{m(T)}}}{K}} \sum_{m=1}^{m(T)}
    \frac{\tau_m - \tau_{m-1}}{\sqrt{\tau_m}}
    \\
    & \leq
    \sqrt{\frac{d_{\tau_{m(T)}}}{K}} \sum_{m=1}^{m(T)}
    \int_{\tau_{m-1}}^{\tau_m} \frac{dx}{\sqrt{x}}
    =
    \sqrt{\frac{d_{\tau_{m(T)}}}{K}}
    \int_{\tau_0}^{\tau_{m(T)}} \frac{dx}{\sqrt{x}}
    = 2\sqrt{\frac{d_{\tau_{m(T)}}}{K}} \sqrt{\tau_{m(T)}}
    .
  \end{align*}
  Above, the first step uses the fact that $m(1) = 1$ and $\tau_{m(t)-1}+1
  \leq t \leq \tau_{m(t)}$.
  The second step uses the definition of $\mu_m$.
  The third step simplifies the sum over $t$ and uses the bound
  $d_{\tau_{m-1}} \leq d_{\tau_{m(T)}}$.
  The remaining steps use an integral bound which is then directly
  evaluated (recalling that $\tau_0 = 0$).
\end{proof}

\begin{lemma}
\label{lem:sum-mu'}
For any $T \in \N$,
\begin{equation*}
  \sum_{t=1}^T \mu_{m(t)-1} \leq
  \frac{\tau_{m_0}}{2K}
  + \sqrt{\frac{8d_{\tau_{m(T)}} \tau_{m(T)}}{K}}
  .
\end{equation*}
\end{lemma}
\begin{proof}
  Under the epoch schedule condition $\tau_{m+1} \leq
  2\tau_m$, we have $\mu_{m(t)-1} \leq \sqrt2 \mu_{m(t)}$ whenever $m(t) >
  m_0$; also, $\mu_{m(t)-1} \leq 1/(2K)$ whenever $m(t) \leq m_0$.
  The conclusion follows by applying \lemref{sum-mu}.
\end{proof}

\begin{lemma}
\label{lem:regret}
For any $T \in \N$, with probability at least $1-\delta$, the regret after
$T$ rounds is at most
\begin{equation*}
  \biggerc \biggl(
   4Kd_{\tau_{m_0-1}} 
  + \sqrt{8Kd_{\tau_{m(T)}} \tau_{m(T)}}
  \biggr)
  + \sqrt{8T\log(2/\delta)}
\end{equation*}
where $\biggerc := (4\vlc + \bigc)$ and
$\bigc$ is defined in \lemref{inductive}.
\end{lemma}
\begin{proof}
Fix $T \in \N$.
For each round $t \in \N$, let $Z_t := r_t(\piopt(x_t)) - r_t(a_t) - 
\sum_{\pi \in \Pi} \wt{Q}_{m(t)-1} \Reg(\pi)$.
Since
\begin{equation*}
\E[ r_t(\piopt(x_t)) - r_t(a_t) |H_{t-1} ]
= \Re(\piopt) - \sum_{\pi \in \Pi} \wt{Q}_{m(t)-1}(\pi) \Re(\pi)
= \sum_{\pi \in \Pi} \wt{Q}_{m(t)-1} \Reg(\pi) ,
\end{equation*}
it follows that $\E[Z_t|H_{t-1}] = 0$.
Since $|Z_t| \leq 2$, it follows by Azuma's inequality that
\begin{equation*}
  \sum_{t=1}^T Z_t \leq 2\sqrt{2T\ln(2/\delta)}
\end{equation*}
with probability at least $1-\delta/2$.
By \lemref{variance-bounds}, \lemref{reward-estimates}, and a union bound,
the event $\GoodEvent$ holds with probability at least $1-\delta/2$.
Hence, by another union bound, with probability at least $1-\delta$,
event $\GoodEvent$ holds and the regret of the algorithm is bounded by
\begin{equation*}
  \sum_{t=1}^T \sum_{\pi \in \Pi} \wt{Q}_{m(t)-1}(\pi) \Reg(\pi)
+ 2\sqrt{2T\ln(2/\delta)}
.
\end{equation*}
The double summation above is bounded by \lemref{expected-regret} and
\lemref{sum-mu'}:
\begin{equation*}
  \sum_{t=1}^T \sum_{\pi \in \Pi} \wt{Q}_{m(t)-1}(\pi) \Reg(\pi)
  \leq (4\vlc + \bigc) K \sum_{t=1}^T \mu_{m(t)-1}
  \leq
  (4\vlc + \bigc) \biggl(
  \frac{\tau_{m_0}}{2}
  + \sqrt{8Kd_{\tau_{m(T)}} \tau_{m(T)}}
  \biggr)
  .
\end{equation*}
By the definition of $m_0$,
$\tau_{m_0-1} \leq 4Kd_{\tau_{m_0-1}}$.
Since $\tau_{m_0} \leq 2\tau_{m_0-1}$ by assumption, it follows that
$\tau_{m_0}
\leq 8Kd_{\tau_{m_0-1}}$.
\end{proof}

\thmref{regret-main} follows from \lemref{regret} and the fact that
$\tau_{m(T)} \leq 2(T-1)$ whenever $\tau_{m(T)-1} \geq 1$.

There is one last result implied by \lemref{vmax} and \lemref{inductive}
that is used elsewhere.
\begin{lemma}
  \label{lem:stupid}
  Assume event $\GoodEvent$ holds, and $t$ is such that $d_{\tau_{m(t)-1}}
  / \tau_{m(t)-1} \leq 1/(4K)$.
  Then
  \begin{equation*}
    \wh\Re_t(\pi_t)
    \leq \Re(\piopt) + \paren{ \theta_1 + \frac{\bigc}{\theta_2} +
      \bigc +
    1 } K\mu_{m(t)-1} .
  \end{equation*}
\end{lemma}
\begin{proof}
  Let $m' < m(t)$ achieve the $\max$ in the
  definition of $\Vmax{t}(\piopt)$.
  If $\mu_{m'} < 1/(2K)$, then $m' \geq m_0$, and
  \begin{align*}
    \Vmax{t}(\piopt)
    & \leq \theta_1 K + \frac{\wh\Reg_{\tau_{m'}}(\piopt)}
    {\theta_2 \mu_{m'}}
    \\
    & \leq \theta_1 K + \frac{2 \Reg(\piopt) + \bigc K\mu_{m'}}
    {\theta_2 \mu_{m'}}
    = c K
  \end{align*}
  for $c := \theta_1 + \bigc/\theta_2$.
  Above, the second inequality follows by \lemref{inductive}.
  If $\mu_{m'} = 1/(2K)$, then the same bound also holds.
  Using this bound, we obtain from \eqref{event2},
  \begin{equation*}
    \wh\Re_t(\piopt) - \Re(\piopt)
    \leq c K \mu_{m(t)-1} + \frac{d_t}{t\mu_{m(t)-1}}
    .
  \end{equation*}
  To conclude,
  \begin{align*}
    \wh\Re_t(\pi_{\tau_m})
    & = \Re(\piopt)
    + \paren{\wh\Re_t(\piopt) - \Re(\piopt)}
    + \wh\Reg_t(\piopt)
    \\
    & \leq \Re(\piopt)
    + c K \mu_{m(t)-1} + \frac{d_t}{t\mu_{m(t)-1}}
    + \wh\Reg_t(\piopt)
    \\
    & \leq \Re(\piopt)
    + c K \mu_{m(t)-1} + \frac{d_t}{t\mu_{m(t)-1}}
    + \bigc K \mu_{m(t)}
  \end{align*}
  where the last inequality follows from \lemref{inductive}.
  The claim follows because $d_t/t \leq d_{\tau_{m(t)-1}} / \tau_{m(t)-1}$
  and $\mu_{m(t)} \leq \mu_{m(t)-1}$.
\end{proof}

\section{Details of Optimization Analysis}
\label{app:optim}

\subsection{Proof of \lemref{opt-correct}}

Following the execution of \stepref{1}, we must have
\begin{equation}  \label{eq:d1}
 \sum_\pi Q(\pi) (2K+\bpi) \leq 2K.
\end{equation}
This is because, if the condition in \stepref{violating-policy} does not hold, then
\eqref{d1} is already true.
Otherwise, $Q$ is replaced by $Q'=cQ$, and for this set of weights,
\eqref{d1} in fact holds with equality.
Note that, since all quantities are nonnegative, \eqref{d1}
immediately implies both \eqref{reg-cons}, and that
$\sum_\pi Q(\pi) \leq 1$.

Furthermore, at the point where the algorithm halts at
\stepref{2a}, it must be that for all policies $\pi$,
$\deriv{\pi}{Q}\leq 0$.
However, unraveling definitions, we can see that this
is exactly equivalent to \eqref{var-cons}.
\qed

\subsection{Proof of \lemref{scale-pot}}

Consider the function
\[g(c) = B_0 \pot{m}(cQ),\]
where, in this proof, $B_0=2K/(\tau\mu)$, where we recall that we drop
the subscripts on $\tau_m$ and $\mu_m$.  Let $\Qmc(a|x) = (1 -
K\mu)cQ(a|x) + \mu$.  By the chain rule, the first derivative of $g$
is:
\begin{eqnarray}
g'(c) &=&
B_0 \sum_\pi Q(\pi) \frac{\partial g(cQ)}{\partial Q(\pi)}
\nonumber
\\
&=&
\sum_\pi Q(\pi) \biggl((2K + \bpi)
       - 2\empexpx\brackets{ \frac{1}{\Qmc(\pi(x)|x)} }
                   \biggr)
\label{eq:p4}
\end{eqnarray}
To handle the second term, note that
\begin{eqnarray}
\sum_\pi Q(\pi) \empexpx\brackets{ \frac{1}{\Qmc(\pi(x)|x)} }
&=&
\sum_\pi Q(\pi)
        \empexpx\brackets{ 
                \sum_{a\in A} \frac{\1{\pi(x)=a}}{\Qmc(a|x)}
                         }
\nonumber
\\
&=&
\empexpx\brackets{\sum_{a\in A}
                   \sum_\pi \frac{Q(\pi) \1{\pi(x)=a}}{\Qmc(a|x)}
                        }
\nonumber
\\
&=&
\empexpx\brackets{\sum_{a\in A}
                                \frac{Q(a|x)}{\Qmc(a|x)}}
\nonumber
\\
&=&
\frac{1}{c} \empexpx\brackets{\sum_{a\in A}
                               \frac{cQ(a|x)}
                                   {(1-K\mu)cQ(a|x)+\mu}}
\leq \frac{K}{c}.
\label{eq:d5a}
\end{eqnarray}
To see the inequality in \eqref{d5a}, let us fix $x$ and define
$q_a=c Q(a|x)$.
Then
$
\sum_a q_a = c \sum_\pi Q(\pi) \leq 1
$
by \eqref{d3}.
Further, the expression inside the expectation in \eqref{d5a} is
equal to
\begin{eqnarray}
\sum_a \frac{q_a}{(1-K\mu)q_a + \mu}
&=&
K \cdot \frac{1}{K}
   \sum_a \frac{1}{(1-K\mu) + \mu/q_a}
\nonumber
\\
&\leq&
K \cdot \frac{1}{(1-K\mu) + K \mu / \sum_a q_a}
\label{eq:5}
\\
&\leq&
K \cdot \frac{1}{(1-K\mu) + K \mu} = K.
\label{eq:6}
\end{eqnarray}
\eqref{5} uses Jensen's inequality, combined with the fact that the
function
$1/(1-K\mu + \mu/x)$ is concave (as a function of $x$).
\eqref{6} uses the fact that the function
$1/(1-K\mu + K\mu/x)$ is nondecreasing (in $x$), and that
the $q_a$'s sum to at most $1$.

Thus, plugging \eqref{d5a} into \eqref{p4} yields
\[
   g'(c)\geq \sum_\pi Q(\pi) (2K + \bpi) - \frac{2K}{c} = 0
\]
by our definition of $c$.
Since $g$ is convex, this means that $g$ is nondecreasing for all
values exceeding $c$.
In particular, since $c<1$, this gives
\[
  B_0 \pot{m}(Q) = g(1)
       \geq g(c)
         = B_0 \pot{m}(cQ),
\]
implying the lemma since $B_0>0$.
\qed

\subsection{Proof of \lemref{pot-dec}}

We first compute the change in potential for general $\alpha$.
Note that $\Qpm(a|x)=\Qm(a|x)$ if $a\neq \pi(x)$, and otherwise
\[ \Qpm(\pi(x)|x) = \Qm(\pi(x)|x) + (1-K\mu)\alpha. \]
Thus, most of the terms defining $\pot{m}(Q)$ are left unchanged by
the update.
In particular, by a direct calculation:
\begin{eqnarray}
\frac{2K}{{\tau\mu}}(\pot{m}(Q) - \pot{m}(Q'))
&=&
\frac{2}{1-K\mu}
\empexpx\brackets{\ln\paren{1 + \frac{\alpha(1-K\mu)}{\Qm(\pi(x)|x)}}}
             -\alpha (2K+ \bpi)
\nonumber
\\
&\geq&
\frac{2}{1-K\mu}
\empexpx\brackets{
              \frac{\alpha(1-K\mu)}{\Qm(\pi(x)|x)}
              - \frac{1}{2} \paren{\frac{\alpha(1-K\mu)}{\Qm(\pi(x)|x)}}^2}
\nonumber
\\
&&             -\alpha (2K+ \bpi)
\label{eq:c1}
\\
&=&
2\alpha \varp{\pi}{Q}
- {(1-K\mu)\alpha^2} \varsq{\pi}{Q}
             -\alpha (2K+ \bpi)
\nonumber
\\
&=&
\alpha (\varp{\pi}{Q} + \deriv{\pi}{Q})
- {(1-K\mu)\alpha^2} \varsq{\pi}{Q}
\label{eq:a4}
\\
&=&
\frac{(\varp{\pi}{Q}+\deriv{\pi}{Q})^2}{4(1-K\mu)\varsq{\pi}{Q}}.
\label{eq:a4b}
\end{eqnarray}
\eqref{c1} uses the bound
$\ln(1+x)\geq x - x^2/2$ which holds for $x\geq 0$ (by Taylor's
theorem).
\eqref{a4b} holds by our choice of $\alpha=\updatep{\pi}{Q}$, which
was chosen to maximize \eqref{a4}.
By assumption, $\deriv{\pi}{Q}>0$, which implies $\varp{\pi}{Q} > 2K$.
Further, since $\Qm(a|x)\geq \mu$ always, we have
\begin{eqnarray*}
  \varsq{\pi}{Q}
     &=& \empexpx\brackets{ \frac{1}{\Qm(\pi(x) \mid x)^2} }
\\
   &\leq& \frac{1}{\mu}
     \cdot \empexpx\brackets{ \frac{1}{\Qm(\pi(x) \mid x)} }
     = \frac{\varp{\pi}{Q}}{\mu}.
\end{eqnarray*}
Thus,
\[
   \frac{(\varp{\pi}{Q}+\deriv{\pi}{Q})^2}{\varsq{\pi}{Q}}
  \geq
   \frac{\varp{\pi}{Q}^2}{\varsq{\pi}{Q}}
  =
   \varp{\pi}{Q} \cdot \frac{\varp{\pi}{Q}}{\varsq{\pi}{Q}}
  \geq
    2K \mu.
\]
Plugging into \eqref{a4b} completes the lemma.
\qed

\subsection{Proof of \lemref{warm-ep2ep-bnd}}

We break the potential of \eqref{pot} into pieces
and bound the total change in each separately.
Specifically, by straightforward algebra, we can write
\[
  \pot{m}(Q) = \pota{m}(Q) + \potb{m} + \potc{m}(Q) + \potd{m}(Q)
\]
where
\begin{eqnarray*}
\pota{m}(Q) &=&
    \frac{\tau_m \mu_m}{K (1-K\mu_m)}
              \empexpx\brackets{-\sum_a \ln \Qm(a | x)}
\\
\potb{m} &=&
    \frac{\tau_m \mu_m \; \ln K}{1-K\mu_m}
\\
\potc{m}(Q) &=&
    \tau_m \mu_m \paren{\sum_\pi Q(\pi) - 1}
\\
\potd{m}(Q) &=&
    \frac{\tau_m \mu_m}{2K} \sum_\pi Q(\pi) \bpi.
\end{eqnarray*}

We assume throughout that $\sum_\pi Q(\pi)\leq 1$ as will always be
the case for the vectors produced by \algref{coord}.
For such a vector $Q$,
\begin{eqnarray*}
\potc{m+1}(Q) - \potc{m}(Q)
    &=&
    (\tau_{m+1} \mu_{m+1} - \tau_m \mu_m) \paren{\sum_\pi Q(\pi) - 1}
   \leq 0
\end{eqnarray*}
since $\tau_m \mu_m$ is nondecreasing.
This means we can essentially disregard the change in this term.

Also, note that $\potb{m}$ does not depend on $Q$.
Therefore, for this term, we get a telescoping sum:
\[
  \sum_{m=1}^M  (\potb{m+1} - \potb{m})
    = \potb{M+1} - \potb{1} \leq \potb{M+1}
     \leq 2 \sqrt{\frac{T d_T}{K}} \ln K
\]
since $K\mu_{M+1}\leq 1/2$, and where $d_T$, used in the definition of
$\mu_m$, is defined in \eqref{dt-def}.

Next, we tackle $\pota{m}$:
\begin{lemma}
\[
  \sum_{m=1}^M  (\pota{m+1}(Q_m) - \pota{m}(Q_m))
     \leq 6 \sqrt{\frac{T d_T}{K}} \ln(1/\mu_{M+1}).
\]
\end{lemma}

\proof
For the purposes of this proof, let
\[
  C_m = \frac{\mu_m}{1-K\mu_{m}}.
\]
Then we can write
\[
  \pota{m}(Q) = -\frac{C_m}{K} \sum_{t=1}^{\tau_m} \sum_a \ln \Qmm(a|x_t).
\]
Note that $C_m\geq C_{m+1}$ since $\mu_m\geq\mu_{m+1}$ and
$-\ln\Qmm(a|x_t) \geq 0$. Thus,
\begin{eqnarray}
\pota{m+1}(Q) - \pota{m}(Q)
&\leq&
  \frac{C_{m+1}}{K}
    \biggl[
           \sum_{t=1}^{\tau_m} \sum_a \ln \Qmm(a|x_t)
\nonumber
\\
&&
         -
           \sum_{t=1}^{\tau_{m+1}} \sum_a \ln \Qmmp(a|x_t)
    \biggr]
\nonumber
\\
&=&
  \frac{C_{m+1}}{K}
    \biggl[
           \sum_{t=1}^{\tau_m} \sum_a \ln\paren{\frac{\Qmm(a|x_t)}{\Qmmp(a|x_t)}}
\nonumber
\\
&&
        -
           \sum_{t=\tau_{m} + 1}^{\tau_{m+1}} \sum_a \ln \Qmmp(a|x_t).
    \biggr]
\nonumber
\\
&\leq&
   C_{m+1} \brackets{ \tau_m \ln(\mu_m / \mu_{m+1})
                     -
                     (\tau_{m+1}-\tau_m) \ln \mu_{m+1}
                   }.
\label{eq:p1}
\end{eqnarray}
\eqref{p1} uses $\Qmmp(a|x)\geq \mu_{m+1}$, and also
\[
  \frac{\Qmm(a|x)}{\Qmmp(a|x)}
 =
  \frac{(1-K\mu_m)Q(a|x)+\mu_m}
       {(1-K\mu_{m+1})Q(a|x)+\mu_{m+1}}
 \leq
  \frac{\mu_m}{\mu_{m+1}},
\]
using $\mu_{m+1} \leq \mu_m$. A sum over the two terms appearing in
\eqref{p1} can now be bounded separately.  Starting with the one on
the left, since $\tau_m<\tau_{m+1}\leq T$ and $K\mu_m\leq 1/2$, we
have
\[
 C_{m+1} \tau_m
   \leq
 2 \tau_m \mu_{m+1}
   \leq
 2 \tau_{m+1} \mu_{m+1}
   \leq
 2 \sqrt{\frac{T d_T}{K}}.
\]
Thus,
\begin{eqnarray}
\sum_{m=1}^M C_{m+1} \tau_m \ln(\mu_m / \mu_{m+1})
&\leq&
 2 \sqrt{\frac{T d_T}{K}}
   \sum_{m=1}^M \ln(\mu_m / \mu_{m+1})
\nonumber
\\
&=&
 2 \sqrt{\frac{T d_T}{K}}
   \ln(\mu_1 / \mu_{M+1}))
\nonumber
\\
&\leq&
 2 \sqrt{\frac{T d_T}{K}}
  (- \ln(\mu_{M+1})).
\label{eq:p2}
\end{eqnarray}

For the second term in \eqref{p1}, using
$\mu_{m+1} \geq \mu_{M+1}$ for $m\leq M$, and
definition of $C_m$,
we have
\begin{eqnarray}
\sum_{m=1}^M -C_{m+1} (\tau_{m+1}-\tau_m) \ln \mu_{m+1}
 &\leq&
    -2 (\ln \mu_{M+1}) \sum_{m=1}^M (\tau_{m+1}-\tau_m) \mu_{m+1}
\nonumber
\\
 &\leq&
    -2 (\ln \mu_{M+1}) \sum_{t=1}^T \mu_{m(t)}
\nonumber
\\
&\leq& 
    -4 \sqrt{\frac{T d_T}{K}} (\ln \mu_{M+1})
\label{eq:p3}
\end{eqnarray}
by \lemref{sum-mu}.
Combining Eqs.~(\ref{eq:p1}),~(\ref{eq:p2}) and~(\ref{eq:p3})
gives the statement of the lemma.
\qed

Finally, we come to $\potd{m}(Q)$, which, by definition of
$\bpi$, can be rewritten as
\[
  \potd{m}(Q) = B_1 \tau_m \sum_\pi Q(\pi) \empreg{\tau_m}(\pi)
\]
where $B_1=1/(2K\vlc)$ and $\vlc$ is the same as appears in
optimization problem (\optprob).
Note that, conveniently,
\[
   \tau_m \empreg{\tau_m}(\pi) = \cumR_m(\pi_m) - \cumR_m(\pi),
\]
where $\cumR_m(\pi)$ is the cumulative empirical importance-weighted
reward through round $\tau_m$:
\[
   \cumR_m(\pi) = \sum_{t=1}^{\tau_m} \instR_t(\pi(x_t))
                = \tau_m \emprew_{\tau_m}(\pi). 
\]

From the definition of $\Qdist$, we have that
\begin{eqnarray*}
\potd{m}(\Qdist)
 &=& \potd{m}(Q)
\\
&& + B_1 \paren{1-\sum_\pi Q(\pi)} \tau_m \empreg{\tau_m}(\pi_m)
\\
 &=& \potd{m}(Q)
\end{eqnarray*}
since $\empreg{\tau_m}(\pi_m)=0$.
And by a similar computation,
$\potd{m+1}(\Qdist) \geq \potd{m+1}(Q)$ since
$\empreg{\tau_{m+1}}(\pi)$ is always nonnegative.

Therefore,
\begin{eqnarray}
\potd{m+1}(Q_m) - \potd{m}(Q_m)
&\leq&
\potd{m+1}(\Qdist_m) - \potd{m}(\Qdist_m)
\nonumber
\\
&=&
B_1 \sum_\pi \Qdistmpi
     \biggl[
      \paren{\cumR_{m+1}(\pi_{m+1}) - \cumR_{m+1}(\pi)}
     -
      \paren{\cumR_{m}(\pi_{m}) - \cumR_{m}(\pi)}
      \biggr]
\nonumber
\\
&=&
B_1  \paren{\cumR_{m+1}(\pi_{m+1}) - \cumR_{m}(\pi_{m})}
\nonumber
\\
&&
 - B_1 \paren{\sum_{t=\tau_m + 1}^{\tau_{m+1}} \sum_\pi \Qdistmpi \instR_{t}(\pi(x_t))}.
\label{eq:e1}
\end{eqnarray}

We separately bound the two parenthesized expressions in \eqref{e1}
when summed over all epochs.
Beginning with the first one, we have
\begin{eqnarray*}
\sum_{m=1}^M 
  \paren{\cumR_{m+1}(\pi_{m+1}) - \cumR_{m}(\pi_{m})}
&=&
  \cumR_{M+1}(\pi_{M+1}) - \cumR_{1}(\pi_{1})
\leq
  \cumR_{M+1}(\pi_{M+1}).
\end{eqnarray*}
But by \lemref{stupid} (and under the same assumptions),
\begin{eqnarray}
  \nonumber \cumR_{M+1}(\pi_{M+1}) &=& \tau_{M+1}
  \emprew_{\tau_{M+1}}(\pi_{M+1})  \\
  &\leq& \tau_{M+1} (\Re(\piopt) + D_0 K \mu_{M}) \nonumber \\ 
  &\leq& \tau_{M+1} \Re(\piopt) + D_0 \sqrt{KT d_T},
  \label{eq:phid-term1}
\end{eqnarray}
where $D_0$ is the constant appearing in \lemref{stupid}.

For the second parenthesized expression of \eqref{e1}, let us define
random variables
\[
  Z_t = \sum_\pi \Qdisttautpi \instR_{t}(\pi(x_t)).
\]
Note that $Z_t$ is nonnegative, and if $m=\tau(t)$, then
\begin{eqnarray*}
Z_t &=&
   \sum_\pi \Qdistmpi \instR_{t}(\pi(x_t))
\\
&=&
\sum_a \Qdist_m(a | x_{t}) \instR_{t}(a)
\\
&=&
\sum_a \Qdist_m(a | x_{t}) \frac{r_{t}(a) \1{a = a_{t}}}
                               {\Qdistmum(a | x_{t})}
\\
&\leq&
\frac{r_{t}(a_{t})}
     {1 - K\mu_m} \leq 2
\end{eqnarray*}
since $\Qdistmum(a|x) \geq (1-K\mu_m) \Qdist_m(a|x)$, and
since $r_t(a_t)\leq 1$ and $K\mu_m\leq 1/2$.
Therefore, by Azuma's inequality, with probability at least $1-\delta$,
\[
   \sum_{t=1}^{\tau_{M+1}} Z_t
   \geq
   \sum_{t=1}^{\tau_{M+1}} \E[Z_t | H_{t-1}]
   - \sqrt{2\tau_{M+1} \ln(1/\delta)}.
\]
The expectation that appears here can be computed to be
\[
    \E[Z_t | H_{t-1}] =
        \sum_\pi \Qdistmpi \Re(\pi)
\]
so
\begin{eqnarray*}
  \Re(\piopt) - \E[Z_t | H_{t-1}] &=&
        \sum_\pi \Qdistmpi (\Re(\piopt) - \Re(\pi))
\\
      &=& \sum_\pi \Qdistmpi \Reg(\pi)
\\
      &\leq& (4\vlc+\bigc) K \mu_m
\end{eqnarray*}
by \lemref{expected-regret} (under the same assumptions, and using
the same constants).
Thus, with high probability,
\begin{eqnarray*}
  \sum_{t=1}^{\tau_{M+1}}
       (\Re(\piopt) - Z_t)
 &\leq&
    (4\vlc+\bigc) K   \sum_{t=1}^{\tau_{M+1}}  \mu_{m(t)}
    +    
    \sqrt{2\tau_{M+1} \ln(1/\delta)}
\\
 &\leq&
    (4\vlc+\bigc) \sqrt{8KT d_T}
    +    
    \sqrt{2 T \ln(1/\delta)}
\end{eqnarray*}
by \lemref{sum-mu'}.

Combining the above bound with our earlier
inequality~\eqref{phid-term1}, and applying the union bound, we find
that with probability at least $1-2\delta$, for all $T$ (and
corresponding $M$),
\[
\sum_{m=1}^M (\potd{m}(Q_m) - \potd{m+1}(Q_m)) \leq
O\paren{\sqrt{\frac{T}{K}\ln(T|\Pi|/\delta)}}.
\]

Combining the bounds on the separate pieces, we get the bound stated
in the lemma.

\subsection{Proof of \lemref{warm-start}}

We finally have all the pieces to establish our main bound on the
oracle complexity with warm-start presented in
\lemref{warm-start}. The proof is almost immediate, and largely
follows the sketch in Section~\ref{sec:warm-start} apart from one
missing bit of detail. Notice that we start Algorithm~\ref{alg:erucb}
with $Q_0 = \zerovec$, at which point the objective $\Phi_0(Q_0) = 0$ since
$\tau_0 = 0$. Initially, owing to the small values of $\tau_m$, we
might be in the regime where $\mu = 1/2K$, where the decrease in the
potential guaranteed by \lemref{pot-dec} is just
$\otil(\tau/K^2)$. However, in this regime, it is easy to check that
$Q_0 = 0$ remains a feasible solution to (\optprob). It clearly
satisfies the regret constraint~\eqref{reg-cons}, and $\mu = 1/(2K)$
ensures that the variance constraints~\eqref{var-cons} are also
met. Hence, we make no calls to the oracle in this initial regime and
can focus our attention to $\tau_m$ large enough so that $\mu_m =
\sqrt{d_{\tau_m}/(K\tau_m)}$.

In this regime, we observe that $\tau_m^2\mu_m = d_{\tau_m}/K$, so
that \lemref{pot-dec} guarantees that we decreaes the objective by at
least $d_{\tau_m}/(4K)$ (recalling $K\mu_m \geq 0$). Hence, the total
decrease in our objective after $N$ calls to the oracle is at least
$Nd_{\tau_m}/(4K)$, while the net increase is bounded by
$\otil(\sqrt{Td_T/K}$. Since the potential is always positive, the
number of oracle calls can be at most $\otil(\sqrt{TK/\logpd})$, which
completes the proof.

\section{Proof of \thmref{lb}}
\label{app:lb}

Recall the earlier definition of the low-variance distribution set
\begin{equation*}
  \lowvar{m} = \{Q \in \Delta^{\Pi} :
  \text{$Q$ satisfies \eqref{var-cons} in round $\tau_m$}\}. 
  \label{eqn:lowvar-set}
\end{equation*}
Fix $\delta \in (0,1)$ and the epoch sequence, and assume $M$ is large
enough so $\mu_m = \sqrt{\ln(16\tau_m^2|\Pi|/\delta)/\tau_m}$ for all $m
\in \N$ with $\tau_m \geq \tau_M/2$.
The low-variance constraint \eqref{var-cons} gives, in round $t=\tau_m$,
\[ \empexpx\left[\frac{1}{Q^{\mu_m}(\pi(x) |
x)}\right] \leq 2K + \frac{\wh\Reg_{\tau_m}(\pi)}{\vlc \mu_m}
, \quad \forall \pi \in \Pi . \]
Below, we use a policy class $\Pi$ where every policy $\pi \in \Pi$ has no
regret ($\Reg(\pi) = 0$), in which case \lemref{inductive} implies
\[ \empexpx\left[\frac{1}{Q^{\mu_m}(\pi(x) | x)}\right] \leq 2K +
\frac{c_0 K \mu_m}{\vlc \mu_m} = K\left(2 + \frac{c_0}{\vlc}\right)
, \quad \forall \pi \in \Pi . \]
Applying \lemref{variance-bounds} (and using our choice of $\mu_m$) gives
the following constraints: with probability at least $1-\delta$,
for all $m \in \N$ with $\tau_m \geq \tau_M/2$,
for all $\pi \in \Pi$,
\begin{equation}
  \E_{x \sim \D_X}\left[\frac{1}{\wt{Q}^{\mu_m}(\pi(x) | x)}\right]
  \ \leq\ 81.3K + 6.4 K\paren{2 + \frac{c_0}{\vlc}}
  =: c K
  \label{eq:truevarcons}
\end{equation}
(to make $Q$ into a probability distribution $\wt{Q}$, the leftover mass
can be put on any policy, say, already in the support of $Q$).
That is, with high probability, for every relevant epoch $m$, every $Q \in
\lowvar{m}$ satisfies \eqref{truevarcons} for all $\pi \in \Pi$.

Next, we construct an instance with the property that these inequalities
cannot be satisfied by a very sparse $Q$.
An instance is drawn uniformly at random from $N$ different contexts
denoted as $\{1,2,\ldots,N\}$
(where we set, with foresight, $N := 1/(2\sqrt2cK\mu_M)$).
The reward structure in the problem will be extremely simple,
with action $K$ always obtaining a reward of 1, while all the other
actions obtain a reward of 0, independent of the context.  The
distribution $\D$ will be uniform over the contexts (with these
deterministic rewards). Our policy set $\Pi$ will consist of $(K-1)N$
separate policies, indexed by $1 \leq i \leq N$ and $1 \leq j \leq
K-1$. Policy $\pi_{ij}$ has the property that
\begin{equation*}
  \pi_{ij}(x)
= \begin{cases}
j & \text{if $x = i$} ,\\
K & \text{otherwise}.
\end{cases}
  \label{eq:policy-lb}
\end{equation*}
In words, policy $\pi_{ij}$ takes action $j$ on context $i$, and action $K$ on all other contexts. Given the uniform distribution over contexts and our reward structure, each policy obtains an identical reward
\[
\Re(\pi) = \left(1 - \frac{1}{N}\right)\cdot 1 + \frac{1}{N}\cdot 0 = 1 - \frac{1}{N}. 
\]
In particular, each policy has a zero expected regret as required. 

Finally, observe that on context $i$, $\pi_{ij}$ is the unique policy
taking action $j$.
Hence we have that $\wt{Q}(j | i) = \wt{Q}(\pi_{ij})$ and $\wt{Q}^{\mu_m}(j
| i) = (1 - K\mu_m)\wt{Q}(\pi_{ij}) + \mu_m$.
Now, let us consider the constraint \eqref{truevarcons} for the policy
$\pi_{ij}$. The left-hand side of this constraint can be simplified as
\begin{align*}
  \E_{x \sim \D_X}\left[\frac{1}{\wt{Q}^{\mu_m}(\pi(x) | x)}\right]
  &= \frac{1}{N} \sum_{x=1}^N \frac{1}{\wt{Q}^{\mu_m}(\pi_{ij}(x) | x)}\\
  &= \frac{1}{N} \sum_{x \ne i} \frac{1}{\wt{Q}^{\mu_m}(\pi_{ij}(x) | x)} +
  \frac{1}{N} \cdot \frac{1}{\wt{Q}^{\mu_m}(j | i)}\\ 
  &\geq \frac{1}{N} \cdot \frac{1}{\wt{Q}^{\mu_m}(j | i)}.
\end{align*}
If the distribution $\wt{Q}$ does not put any support on the policy
$\pi_{ij}$, then $\wt{Q}^{\mu_m}(j | i) = \mu_m$, and thus
\begin{align*}
  \E_{x \sim \D_X}\left[\frac{1}{\wt{Q}^{\mu_m}(\pi(x) | x)}\right]
  &\geq \frac{1}{N} \cdot \frac{1}{\wt{Q}^{\mu_m}(j | i)}
  = \frac{1}{N\mu_m}
  \geq \frac{1}{\sqrt2N\mu_M}
  > cK
\end{align*}
(since $N < 1/(\sqrt2cK\mu_M)$).
Such a distribution $\wt{Q}$ violates \eqref{truevarcons}, which means that
every $Q \in \lowvar{m}$ must have $\wt{Q}(\pi_{ij}) > 0$.
Since this is true for each policy $\pi_{ij}$, we see that every $Q \in
\lowvar{m}$ has
\[
|\supp(Q)|
\geq (K-1)N
= \frac{K-1}{2\sqrt2cK\mu_M}
= \Omega\paren{ \sqrt{\frac{K\tau_M}{\ln(\tau_M|\Pi|/\delta)}} }
\]
which completes the proof.

\section{Online Cover algorithm}
\label{app:experiments}

This section describes the pseudocode of the precise algorithm use in
our experiments (\algref{OO}).
The minimum exploration probability $\mu$ was set as $0.05\,\min(1/K,
1/\sqrt{tK})$ for our evaluation.

\begin{algorithm}[h]
  \caption{Online Cover}
  \label{alg:OO}
  \begin{algorithmic}[1]
    \renewcommand{\algorithmicrequire}{\textbf{input}}

    \REQUIRE Cover size $n$, minimum sampling probability $\mu$.

    \STATE Initialize online cost-sensitive minimization oracles $O_1, O_2,
    \dotsc, O_n$, each of which controls a policy $\pi_{(1)}, \pi_{(2)},
    \dotsc, \pi_{(n)}$; $U := $ uniform probability distribution over these
    policies.

    \FOR{\textbf{round} $t = 1, 2, \dotsc$}

      \STATE Observe context $x_t \in X$.

      \STATE $(a_t,p_t(a_t)) := \SAMPLE(x_t,U,\emptyset,\mu)$.

      \STATE Select action $a_t$ and observe reward $r_t(a_t) \in [0,1]$.

      \FOR{\textbf{each} $i = 1, 2, \dotsc, n$}

        \STATE $Q_i := (i-1)^{-1} \sum_{j<i} \one_{\pi_{(j)}}$.

        \STATE $p_i(a) := Q^\mu(a|x_t)$.

        \STATE Create cost-sensitive example $(x_t,c)$ where $c(a) = 1 -
        \frac{r_t(a_t)}{p_t(a_t)}\ind{a = a_t} - \frac{\mu}{p_i(a)}$.
        \label{step:oc-example}

        \STATE Update $\pi_{(i)} = O_i(x,c)$

      \ENDFOR

    \ENDFOR
  \end{algorithmic}
\end{algorithm}

Two additional details are important in \stepref{oc-example}:

\begin{enumerate}
\item We pass a cost vector rather than a reward vector to the oracle
  since we have a loss minimization rather than a reward maximization
  oracle.
\item We actually used a doubly robust estimate~\cite{DoubleRobust}
  with a linear reward function that was trained in an online fashion.
\end{enumerate}

\end{document}